\documentclass[sigconf]{acmart}

\usepackage{natbib}
\usepackage{subfigure}
\usepackage{booktabs}

\usepackage{amsmath}
\usepackage{mathtools}
\usepackage{amsthm}

\usepackage{algorithm}
\usepackage{algorithmicx}
\usepackage[noend]{algpseudocode}

\theoremstyle{plain}
\newtheorem{theorem}{Theorem}[section]

\newtheorem{lemma}[theorem]{Lemma}

\theoremstyle{definition}
\newtheorem{definition}[theorem]{Definition}

\theoremstyle{remark}

\newtheorem{conjecture}[theorem]{Conjecture}
\newtheorem{example}{Example}
\newtheorem{case}{Case}
\counterwithin{case}{theorem}

\newenvironment{proofsketch}{%
  \renewcommand{\proofname}{Proof Sketch}\proof}{\endproof}

\newcommand*{\argmax}{arg\,max}

\newcommand*{\dprime}{^{\prime\prime}\mkern-1.2mu}
\usepackage{multirow}
\usepackage{wasysym}

\usepackage{booktabs}
\usepackage{siunitx}
\usepackage{lipsum}  
\usepackage{makecell}
\usepackage{bbm}
\newcommand*{\baseline}{$^\S$}
\newcommand*{\fairInduce}{$^\dagger$}
\newcommand*{\fairGuarantee}{$^{\ddag}$}
\newcommand*{\fairAgnostic}{$^\star$}
\usepackage{graphics}

\usepackage{upgreek}

\usepackage{enumerate}
\usepackage{enumitem}

\newcommand{\uuu}{u}

\usepackage{thmtools}
\usepackage{thm-restate}

\usepackage{tikz}
\usetikzlibrary{trees,fit}
\definecolor{cadetBlue}{HTML}{002366}
\definecolor{RoyalBlue}{HTML}{002366}

\AtBeginDocument{%
  \providecommand\BibTeX{{%
    \normalfont B\kern-0.5em{\scshape i\kern-0.25em b}\kern-0.8em\TeX}}}

\copyrightyear{2023}
\acmYear{2023}
\setcopyright{rightsretained}
\acmConference[KDD '23]{Proceedings of the 29th ACM SIGKDD Conference on Knowledge Discovery and Data Mining}{August 6--10, 2023}{Long Beach, CA, USA}
\acmBooktitle{Proceedings of the 29th ACM SIGKDD Conference on Knowledge Discovery and Data Mining (KDD '23), August 6--10, 2023, Long Beach, CA, USA}
\acmDOI{10.1145/3580305.3599467}
\acmISBN{979-8-4007-0103-0/23/08}

\makeatletter
\gdef\@copyrightpermission{
  \begin{minipage}{0.3\columnwidth}
   \href{https://creativecommons.org/licenses/by/4.0/}{\includegraphics[width=0.90\textwidth]{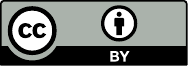}}
  \end{minipage}\hfill
  \begin{minipage}{0.7\columnwidth}
   \href{https://creativecommons.org/licenses/by/4.0/}{This work is licensed under a Creative Commons Attribution International 4.0 License.}
  \end{minipage}
  \vspace{5pt}
}
\makeatother

\begin{document}

\title{Planning to Fairly Allocate: Probabilistic Fairness in the Restless Bandit Setting}

\author{Christine Herlihy}
\authornote{Denotes equal contribution to this research.}
\email{cherlihy@cs.umd.edu}
\orcid{0000-0001-7405-603X}
\author{Aviva Prins}
\authornotemark[1]
\email{aviva@cs.umd.edu}
\orcid{0009-0006-4561-0314}
\affiliation{%
  \institution{University of Maryland}
  \streetaddress{8125 Paint Branch Drive}
  \city{College Park}
  \state{Maryland}
  \country{USA}
  \postcode{20742}
}

\author{Aravind Srinivasan}
\email{asriniv1@umd.edu}
\orcid{0000-0002-0062-3684}
\affiliation{%
  \institution{University of Maryland}
  \streetaddress{8125 Paint Branch Drive}
  \city{College Park}
  \state{Maryland}
  \country{USA}
  \postcode{20742}
}

\author{John P. Dickerson}
\email{johnd@umd.edu}
\orcid{0000-0003-2231-680X}
\affiliation{%
  \institution{University of Maryland}
  \streetaddress{8125 Paint Branch Drive}
  \city{College Park}
  \state{Maryland}
  \country{USA}
  \postcode{20742}
}

\renewcommand{\shortauthors}{Christine Herlihy, Aviva Prins, Aravind Srinivasan, \& John P. Dickerson}

\begin{abstract}
Restless and collapsing bandits are often used to model budget-constrained resource allocation in settings where arms have action-dependent transition probabilities, such as the allocation of health interventions among patients. However,  SOTA Whittle-index-based approaches to this planning problem either do not consider fairness among arms, or incentivize fairness without guaranteeing it. We thus introduce \textsc{ProbFair}, a probabilistically fair policy that maximizes total expected reward and satisfies the budget constraint while ensuring a strictly positive lower bound on the probability of being pulled at each timestep. 
We evaluate our algorithm on a real-world application, where interventions support continuous positive airway pressure (CPAP) therapy adherence among  patients, as well as on a broader class of synthetic transition matrices. We find that \textsc{ProbFair} preserves utility while providing fairness guarantees.
\end{abstract}




\maketitle

\section{Introduction}
\label{sec:intro} 

Restless multi-armed bandits (RMABs) are used to model budget-constrained resource allocation tasks in which a decision-maker must select a subset of arms (e.g., projects, patients, assets) to receive a beneficial intervention at each timestep, while the state of each arm evolves over time in an action-dependent, Markovian fashion. Such problems are common in healthcare, where clinicians may be tasked with monitoring large, distributed patient populations and determining which individuals to expend scarce resources on so as to maximize total welfare. RMABs have been proposed to determine which inmates should be prioritized to receive hepatitis C treatment in U.S. prisons~\citep{ayer2019prioritizing}, and which tuberculosis patients should receive medication adherence support in India~\citep{mate2020collapsing}.

Current state-of-the-art approaches to solving RMABs rely on the indexing work introduced by \citet{whittle1988restless}. While the Whittle index solves an otherwise PSPACE-complete problem in an asymptotically optimal fashion by decoupling arms~\citep{weber1990index}, it fails to provide any guarantees about how pulls will be distributed \emph{among arms}. 

Though the intervention is canonically assumed to be beneficial for \emph{every} arm, the marginal benefit (i.e., relative increase in the probability of a favorable state transition) varies in accordance with each arm's underlying state transition function. Consequently, Whittle index-based maximization of total expected reward \emph{without regard for distributive fairness} empirically allocates all available interventions to a small subset of arms, ignoring the rest~\citep{prins2020incorporating}.

There are many application domains where a bimodal distributive outcome may be perceived as unfair or undesirable by beneficiaries and decision-makers, thus motivating efforts to incentivize or guarantee distributive fairness. In the aforementioned healthcare examples, resource constraints and variation in transition dynamics interact. A practical consequence is that a majority of patients will \emph{never} receive the beneficial intervention(s) in question. This, in turn, means that their clinical outcomes will be strictly worse in expectation than they would be under a policy that guaranteed a non-zero probability of receiving the intervention at each timestep. 

To improve distributive fairness, we explore whether it is possible to modify the Whittle index to guarantee each arm at least one pull per user-defined time interval, but find this to be intractable. We then introduce \textsc{ProbFair}, a state-agnostic policy that maps each arm to a fairness-constraint satisfying, stationary probability distribution over actions that takes the arm's\ transition matrix into account. At each timestep, we then use a dependent rounding algorithm~\citep{srinivasan2001distributions} to sample from this probabilistic policy to produce a budget-constraint satisfying 
discrete action vector. 

We evaluate \textsc{ProbFair} on a randomly generated dataset and a realistic dataset derived from obstructive sleep apnea patients tasked with nightly self-administration of continuous positive airway pressure (CPAP) therapy~\citep{kang2013markov, kang2016modelling}. 

Our core contributions include:
\begin{enumerate}[label=(\textbf{\roman*})]
    \item A novel approach that is both efficiently computable and reward maximizing, subject to the guaranteed satisfaction of budget \emph{and} probabilistic fairness constraints. 
    \item Empirical results demonstrating that \textsc{ProbFair} is competitive vis-\`{a}-vis other fairness-inducing policies, and stable over a range of cohort composition scenarios.
\end{enumerate}

\section{Restless Multi-Armed Bandit Model}
\label{sec:rmabModel}
Here, we give an overview of the restless multi-armed bandit (RMAB) framework, along with our proposed extension, which takes the form of a fairness-motivated constraint. A restless multi-armed bandit consists of \(N\!\in\! \mathbb{N}\) independent arms, each of which evolves over a finite time horizon \(T\!\in\! \mathbb{N}\), according to an associated Markov Decision Process (MDP). Each arm's MDP is characterized by a 4-tuple \((\mathcal{S}, \mathcal{A}, P, r)\) where \(\mathcal{S}\) represents the state space, \(\mathcal{A}\) represents the action space, \(P\) represents an \(\lvert\mathcal{S}\rvert\!\times\! \lvert\mathcal{A}\rvert \!\times\! \lvert\mathcal{S}\rvert\) transition matrix, and \(r: \mathcal{S} \rightarrow \mathbb{R}\) represents a local reward function that maps states to real-valued rewards. Appendix~\ref{sec:notation} summarizes notation; note that $[N]$ denotes the set $\{1, 2, \ldots, N\}$. 

\par \textbf{States, actions, and observability:} We specifically consider a discrete two-state system \(\mathcal{S}\coloneqq\{0,1\}\) where 1 (0) represents being in the ``good'' (``bad'') state, and a set of two possible actions \(\mathcal{A}\coloneqq\{0,1\}\) where 1 represents the decision to select (``pull'') arm $i \in [N]$ at time $t \in [T]$, and 0 represents the choice to be passive (not pull). In the general RMAB setting, each arm's state \(s_t^i\) is observable. We consider the partially-observable extension introduced by \citet{mate2020collapsing}, where arms' states are only observable when they are pulled. Otherwise, an arms' state is replaced with the probabilistic \textit{belief} \(b_t^i\in [0,1]\) that it is in state \(1\). Such partial observability captures uncertainty regarding patient status and treatment efficacy associated with outpatient or 
remotely-administered interventions.

\par \textbf{Transition matrices}: Each arm \(i\) is characterized by a set of transition matrices \(P\), where \(P_{s,s'}^{a} \) represents the probability of transitioning from state \(s\) to state \(s'\) when action \(a\) is taken. We assume \(P\) to be (a) static and (b) known by the agent at planning time. Assumptions (a) and (b) are likely to be violated in practice; however, they provide a useful modeling foundation, and can be modified to incorporate additional uncertainty, such as the requirement that transition matrices must be learned~\citep{jung2019regret}. 
Clinical researchers often use longitudinal data to construct risk-adjusted transition matrices that encode cohort-specific transition probabilities. These can guide patient-level decision-making~\citep{steimle2017markov}. 

Consistent with previous literature, we assume strictly positive transition matrix entries, and impose four \emph{structural constraints}\label{eq:sc}: (a)~$P_{0,1}^0 <  P_{1,1}^0$; (b)~$P_{0,1}^1 <  P_{1,1}^1$; (c)~$P_{0,1}^0 < P_{0,1}^1$; (d)~$P_{1,1}^0 <  P_{1,1}^1$ \citep{mate2020collapsing}. These constraints are application-motivated, and imply that arms are more likely to remain in a ``good'' state than change from a bad state to a good one, and that a pull is helpful when received. In the absence of such constraints, the effect of the intervention may be superfluous or harmful, rather than desirable.

\par \textbf{Objective and constraints}: In the canonical RMAB setting, the agent's goal is to find a policy $\pi^*$ that maximizes total expected reward $\arg \max_\pi \mathbb{E}_{\pi}[R(r(s))]$ while satisfying a \emph{budget constraint}, \(k \ll N \in \mathbb{N}\), which allows the agent to select at most $k$ arms at each timestep. We consider a cumulative reward function, \(R(\cdot)\coloneqq~\sum_{i\in[N]} \sum_{t\in[T]} \beta^{t-1} r(s^i_t)\), for some discount rate $\beta \in [0,1]$, and non-decreasing \(r(s)\). 

We extend this model by introducing a Boolean-valued, distributive fairness-motivated constraint, which may take one of two general forms: 
\begin{enumerate} 
    \item \emph{Time-indexed}: A function $g\left(\cup_{t \in [T]} \{\vec{a}_t\}\right)$ which is satisfied if each arm is pulled at least once within each user-defined time interval $\nu \leq T$ (e.g., at least once every seven days), or a minimum fraction $\psi \in (0,1)$ of times over the entire time horizon~\citep{li2019combinatorial}.
    \item \emph{Probabilistic}: A function $g^\prime( \vec{p}^i | \vec{a}_t  \sim \vec{p}^{i} ~\forall t)$ which operates on the stationary probability vector $\vec{p}^i$, from which discrete actions are drawn, by requiring the probability that each arm receives a pull at any given $t$ to fall within an interval $[\ell, \uuu]$ where $0~<~\ell \leq \frac{k}{N} \leq \uuu \leq 1$. %(an extension of \citep{chenFairContext}).
\end{enumerate}

\section{Context, Motivation \& Related Work}
In this section, we motivate our ultimate focus on probabilistic fairness by revisiting the distribution of pulls under Whittle-index based policies. We begin by providing background information on the Whittle index, and then proceed to ask: (1) Which arms are ignored, and why does it matter? (2) Is it possible to modify the Whittle index so as to provide a \emph{time-indexed fairness guarantee} for each arm? In response to the latter, we demonstrate that time-indexed fairness guarantees necessitate the coupling of arms, which undermines the indexability of the problem. We then identify prior work at the intersection of algorithmic fairness, constrained resource allocation, and multi-armed bandits, and identify desiderata that characterize our own approach. 

\label{sec:relatedWorks}
\subsection{Background: Whittle Index-based Policies} 
Pre-computing the optimal policy for a given set of restless or collapsing arms is PSPACE-hard in the general case~\citep{papadimitriou1994complexity}. However, as established by \citet{whittle1988restless} and formalized by \citet{weber1990index}, if the set of arms associated with a problem are \emph{indexable}, we can decouple the arms and efficiently solve the problem using an asymptotically-optimal heuristic index policy. 

\par \textbf{Mechanics}: At each timestep $t \in~[T]$, the value of a pull, in terms of both immediate and expected discounted future reward, is computed for each decoupled arm, $i \in  [N]$. This value-computation step relies on the notion of a subsidy, $m$, which can be thought of as the opportunity cost of passivity. Formally, the Whittle index is the subsidy required to make the agent indifferent between \emph{pulling} and \emph{not pulling} arm $i$ at time $t$. (Per Section~\ref{sec:rmabModel}, \(b\) denotes the probabilistic belief that an arm is in state \(s=1\); for restless arms, $b_t^i = s_t^i \in \{0,1\}$).
\begin{equation}
\label{eqn:whittleIndex}
    W(b_t^i) = \inf_m \left\{m \mid V_m(b_t^i, a_t^i = 0) \geq V_m(b_t^i, a_t^i = 1)\right\}
\end{equation}
The value function \(V_m(b)\) represents the maximum expected discounted reward under passive subsidy \(m\) and discount rate $\beta$ for arm \(i\) with belief state $b_t^i \in [0,1]$ at time \(t\):
\begin{equation*} \label{eqn:valueFunc}
    V_m(b_t^i) = \max \begin{cases}
    m + r(b_t^i) + \beta V_m \left(b_{t+1}^i\right) \ \textit{passive} \\
    r(b_t^i)+\beta \left[b_t^i V_m\left(P^1_{1,1}\right) + (1-b_t^i) V_m\left(P^1_{0,1}\right)\right] \\ \hspace{150pt}\textit{active}
    \end{cases}
\end{equation*}
Once the Whittle index has been computed for each arm, the agent sorts the indices, and the $k$ arms with the greatest index values receive a pull at time $t$, while the remaining $N-k$ arms are passive. \citet{weber1990index} give sufficient conditions for \emph{indexability}:
\begin{definition} % (Indexability)
\label{def:indexability}
An arm is indexable if the set of beliefs for which it is optimal to be passive for a given $m$, \(\mathcal{B}^*(m) = \{b \mid \forall \pi \in \Pi^*_m, \pi(b)=0\}\), monotonically increases from \(\emptyset\) to the entire belief space as \(m\) increases from \(-\infty\) to \(+\infty\). An RMAB is indexable if every arm is indexable.
\end{definition}

\par \textbf{Indexability} is often difficult to establish, and computing the Whittle index can be complex~\citep{liu2008indexability}. Prevailing approaches rely on proving the optimality of a \emph{threshold policy} for a subset of transition matrices~\citep{nino2020verification}. A \emph{forward} threshold policy pulls an arm when its state is at or below a given threshold, and makes the arm passive otherwise; the converse is true for a \emph{reverse} threshold policy. \citet{mate2020collapsing} give such conditions for this RMAB setting, when \(r(b)=b\), and provide an algorithm, \textsc{Threshold Whittle}, that is asymptotically optimal for forward threshold-optimal arms. \citet{mate2021risk-aware} expand on this work for any non-decreasing \(r(b)\) and present the \textsc{Risk-Aware Whittle} algorithm. 

\subsection{Motivation: Individual Welfare \& Whittle}
\label{sec:motivation}
\par \textbf{Bimodal allocation}: Existing theory does not offer any guarantees about how the sequence of actions will be distributed over arms under Whittle index-based policies, nor about the probability with which a given arm can expect to be pulled at any particular timestep. \citet{prins2020incorporating} demonstrate that Whittle-based policies tend to allocate all pulls to a small number of arms, neglecting most of the population. We present similar findings in Appendix~\ref{sec:AppWhittleUnFair}.

This bimodal distribution is a consequence of how the Whittle index prioritizes arms. Whittle favors arms for whom a pull is most beneficial to achieving sustained occupancy in the ``good'' state, regardless of whether this results in the same subset of arms repeatedly receiving pulls. While the structural constraints in Sec.~\ref{sec:rmabModel} ensure that a pull is beneficial for every arm, marginal benefit varies. Since reward is a function of each arm's underlying state, arms whose trajectories are characterized by a relative---\emph{but not absolute}---indifference to the intervention are likely to be ignored.

\par \textbf{Ethical implications}: This zero-valued lower bound on the number of pulls an arm can receive aligns with a \emph{utilitarian} approach to distributive justice, in which the decision-maker seeks to allocate resources so as to maximize total expected utility~\citep{bentham1781introduction, marseille2019utilitarianism}. This may be incompatible with competing pragmatic and ethical desiderata, including \emph{egalitarian} and \emph{prioritarian} notions of distributive fairness, in which the decision-maker seeks
to allocate resources equally among arms (e.g., \textsc{Round-Robin}), or prioritize arms considered to be worst-off under the status quo, for some quantifiable notion of \emph{worst-off} that induces a partial ordering over arms~\citep{rawls1971theory, scheunemann2011ethics}. We consider the \emph{worst off} to be arms who would be deprived of algorithmic attention (e.g., not receive any pulls), or, from a probabilistic perspective, would have a \emph{zero-valued lower bound} on the probability of receiving a pull at any given timestep. 

\par \textbf{Why algorithmic attention?} This choice is motivated by our desire to improve \emph{equality of opportunity} (i.e., access to the beneficial intervention) rather than \emph{equality of outcomes} (i.e., observed adherence). The agent directly controls who receives the intervention, but has only indirect control (via actions) over the sequence of state transitions an arm experiences. Additionally, proclivity towards adherence may vary widely in the absence of restrictive assumptions about cohort homogeneity, and focusing on equality of outcomes could thus entail a significant loss of total welfare. 

\par \textbf{Distributive fairness and algorithmic acceptability}: To realize the benefits associated with an algorithmically-derived resource allocation policy, practitioners tasked with implementation must find the policy to be acceptable (i.e., in keeping with their professional and ethical standards), and potential beneficiaries must find participation to be rational. 

With respect to \emph{practitioners}, many clinicians report experiencing mental anguish when resource constraints force them to categorically deny a patient access to a beneficial treatment, and may resort to providing improvised and/or sub-optimal care~\citep{butler2020us}. Providing fairness-aware decision support can improve acceptability~\citep{rajkomar2018ensuring, kelly2019key} and minimize the loss of reward associated with ethically-motivated deviation to a sub-optimal but equitable approach such as \textsc{Round-Robin}~\citep{de2020case, dietvorst2015algorithm}. For \emph{beneficiaries}, we posit that an arm may consider participation rational when it results in an increase in expected time spent in the adherent state relative to non-participation (e.g., due to receiving a strictly positive number of pulls in expectation). 

\subsection{Time-indexed Fairness and Indexability}
\label{sec:fairnessIndexability}
We now consider whether it is possible to modify the Whittle index to guarantee time-indexed fairness while preserving our ability to decouple arms. Unfortunately, the answer is no---we provide an overview here and a detailed discussion in Appendix~\ref{sec:altModifyWhittle}. Recall that structural constraints ensure that when an arm is considered in isolation, the optimal action will \emph{always} be to pull, and that a Whittle-index approach computes the infimum subsidy, $m$, an arm requires to accept passivity at time $t$. Whether or not arm $i$ is \emph{actually} pulled at time $t$ depends on how the subsidy of one arm \emph{compares} to the infimum subsidies required by other arms. Thus, any modification intended to \emph{guarantee} time-indexed fairness must be able to alter the ordering \emph{among} arms, such that any arm $i$ which would otherwise have a subsidy with rank $>k$ when sorted in descending order will now be in the top-$k$ arms. Even if we could construct such a modification for a single arm without requiring time-stamped system information, if \emph{every} arm had this same capability, then a new challenge would arise: we would be unable to distinguish among arms, and arbitrary tie-breaking could again jeopardize fairness constraint satisfaction. 

\subsection{Additional Related Work}

While multi-armed bandit problems are canonically framed from the perspective of the decision-maker, interest in individual and group fairness in this setting has grown in recent years~\citep{joseph2016fairness, chenFairContext, li2019combinatorial}.

In the \emph{stochastic} multi-armed bandit setting, each arm is characterized by a fixed but unknown average reward rather than by an MDP. The decision-maker thus faces uncertainty about the true utility of each arm and must balance exploration (i.e., pulling arms to gain information about their reward distributions) with exploitation (i.e., pulling the optimal arm(s)) to maximize expected reward. \citet{joseph2016fairness} examine fairness among arms in this setting, and introduce a definition that requires the decision-maker to favor (i.e., select) arms with higher average reward over arms with lower average reward, even in the face of uncertainty. As the authors note, this definition \emph{is} consistent with reward maximization, but imposes a cost in terms of per-round regret when \emph{learning} the optimal policy, due to the fact that arms with overlapping confidence intervals are chained until they can be separated with high confidence. 

Prior work in other non-restless bandit settings demonstrates that alternative definitions---i.e., those which center \emph{distributive fairness among arms} as opposed to the principle that arms with similar average rewards should be treated similarly~\citep{dwork2011fairness}, generally entail deviation from optimal behavior. \citet{li2019combinatorial} study the combinatorial \emph{sleeping} bandit setting, in which arms are stochastic but may be unavailable at any given timestep. They introduce the minimum selection fraction constraint, which we adapt and refer to as time-indexed fairness (see Section \ref{sec:rmabModel}). \citet{chenFairContext} consider the \textit{contextual} bandit setting, and propose an algorithm that guarantees each arm a minimum probability of selection at each timestep.

In the \emph{restless} setting that we consider, prior works have tended toward opposite ends of the reward-fairness spectrum by either: (1) redistributing pulls without providing arm-level guarantees~\citep{mate2021risk-aware, li2022towards}; or (2) guaranteeing time-indexed fairness without providing optimality guarantees~\citep{prins2020incorporating}.  Recent work has also considered the adjacent problem of fairness among intervention \emph{providers} (i.e., workers)~\citep{biswasfairness}. In contrast to prior work, we aim to \emph{guarantee} rather than
incentivize fairness, without incurring an exponential dependency on the time horizon or sacrificing optimality guarantees. We thus seek an efficient policy that is reward maximizing, subject to the satisfaction of both budget and
probabilistic fairness constraints.

\section{Methodological Approach}
\label{sec:stationary-policy}
Here we introduce \textsc{ProbFair}, an approximately optimal solution to a relaxed version of the allocation task in which we guarantee the satisfaction of \emph{probabilistic} rather than \emph{time-indexed} fairness, along with the budget constraint. This relaxation is necessary for tractability, as it allows us to precompute a stationary, \emph{state-agnostic} probability vector, $\vec{p}^i$, from which constraint-satisfying discrete actions are drawn. 

\textsc{ProbFair} maps each arm $i$ to an arm-specific, stationary probability distribution over atomic actions, such that for each timestep $t$, $P[a^i_t = 1]~=~p_i$ and $P[a^i_t=0]~=~1-p_i$, where $p_i \in [\ell, \uuu]$ for all $i \in [N]$ and $\sum_i p_i = k$. Here, $\ell$ and $\uuu$ are user-defined fairness parameters satisfying $0 < \ell \leq \frac{k}{N} \leq \uuu \leq 1$, per Section~\ref{sec:rmabModel}. Note that $\ell T$ and $\uuu T$ can be interpreted as lower and upper bounds on the expected number of pulls an arm will receive over the time horizon. %, $T$.

In Section~\ref{sec:algApproach}, we describe how to construct the $p_i$'s so as to efficiently approximate our constrained reward-maximization
objective within a multiplicative factor of $(1 - \epsilon)$, for any given constant $\epsilon > 0$. We use a dependent rounding approach detailed in Section~\ref{sec:sampling} to sample from this distribution at each timestep \emph{independently}, to produce a discrete action vector, $\vec{a}_t \in \{0,1\}^N$, which is guaranteed to satisfy the budget constraint, $k$~\citep{srinivasan2001distributions}.

To motivate our approach, note that when we take the union of each arm's stationary probability vector, we obtain a system-level policy, $\pi_{PF}: \{i \mid i \in N\}~\rightarrow~\left[1-p_i, p_i\right]^N$. Regardless of the system's initial state, repeated application of this policy will result in convergence to a steady-state distribution in which (WLOG) arm $i$ is in the adherent state (i.e., state 1) with probability $x_i \in [\ell, \uuu]$, and the non-adherent state (i.e., state 0) with probability $1-x_i \in [0,1]$. 

By definition, for any arm $i$, $x_i$ will satisfy the equation:
\begin{equation}
    x_i\left[(1-p_i)P_{1,1}^{0} + p_iP_{1,1}^{1}\right] + (1-x_i)[(1-p_i)P_{0,1}^{0} + p_iP_{0,1}^{1}] = x_i.
\end{equation}
Thus, $x_i = f_i(p_i)$, where
\begin{equation}
\label{eq:fi}
    f_i(p_i) = \frac{(1-p_i)P_{0,1}^{0} + p_iP_{0,1}^{1}}{1 - (1-p_i)P_{1,1}^{0} - p_iP_{1,1}^{1} + (1-p_i)P_{0,1}^{0} + p_iP_{0,1}^{1}}
\end{equation}

We seek the policy which maximizes total expected reward, where reward is non-decreasing in $s$ (i.e., with time spent in the adherent state). Thus, \textsc{ProbFair} is defined as:
\begin{align}
    \pi_{\textsc{PF}} &= \argmax_{p_i\in [\ell, \uuu]} \sum_i f_i(p_i) \textrm{ s.t. } \sum_i p_i \leq k
\label{eqn:probfair_optimization}
\end{align}
Solving this constrained maximization problem is thus consistent with maximizing the expected number of timesteps each arm will spend in the adherent state, subject to satisfying the budget \emph{and} probabilistic fairness constraints. We emphasize that our construction process takes the transition matrices of each arm $i$ into account via $f_i$ (Equation~\ref{eq:fi}).

\subsection{Computing the $p_i$'s: Algorithmic Approach}
\label{sec:algApproach}
\par \textbf{Overview}: To construct $\pi_{PF}$, we: (1) partition the arms based on the shapes of their respective $f_i$ functions (Eq.~\ref{eq:fi}); (2) perform a grid search over possible ways to allocate the budget, $k$, between the two subsets of arms; (2a) solve each sub-problem to produce a probabilistic policy for the arms in that subset; (2b) compute the total expected reward of the policy; (3) take the argmax over this set of grid search values to determine the approximately optimal budget allocation; and (4) form $\pi_{PF}$ by taking the union over the policies produced by evaluating each sub-problem at its approximately optimal share of the budget. Figure~\ref{fig:probFairVisual} visualizes; the remainder of this section provides technical details.
\begin{figure}[!h]
\centering
   \includegraphics[width=0.95\linewidth]{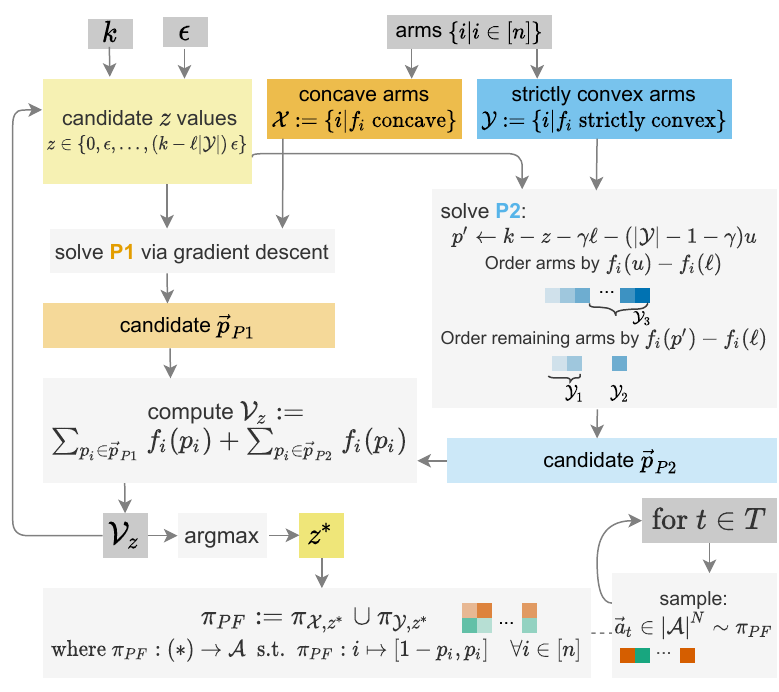}
\caption[TEST]{\textsc{ProbFair}: constructing and sampling from $\pi_{PF}$}
\label{fig:probFairVisual} 
\end{figure}

To begin, we introduce two theorems (see App.~\ref{app:appProbFair} for full proofs):

\begin{restatable}[]{theorem}{fcurvature}\label{thm:fcurvature} 
For every arm $i \in [N]$, $f_i(p_i)$ is either \textbf{concave} or \textbf{strictly convex} in all of $p_i \in [0,1]$.
\end{restatable}

\begin{proofsketch}
WLOG, fix an arm \(i \in [N]\). For notational convenience, let us define the following constants derived from the arm's transition matrix \(P^i\): \(c_1 \coloneqq P_{0,1}^0\), \(c_2 \coloneqq P_{0,1}^1 - P_{0,1}^0\), \(c_3 \coloneqq 1 - P_{1,1}^0 + P_{0,1}^0\), and \(c_4 \coloneqq P_{1,1}^0 - P_{1,1}^1 - P_{0,1}^0 + P_{0,1}^1\). Then
\begin{equation}
    f_i^{\dprime}(p_i) = \frac{2c_4^2 \left(c_1 - \frac{c_2c_3}{c_4}\right)}{(c_3 + c_4p_i)^3}.
\end{equation}
\(c_3 + c_4p_i \in (0,1)\) for all \(p_i \in [0,1]\). Thus, the sign of \(f_i^{\dprime}(p_i)\) is determined by \(c_1 - \frac{c_2c_3}{c_4}\), which does not depend on \(p_i\).
\end{proofsketch}

\begin{restatable}[]{theorem}{fnondecr}\label{thm:fnondecr}
For each arm $i \in [N]$, the structural constraints introduced in Section~\ref{sec:rmabModel} ensure that $f_i(p_i)$ is monotonically non-decreasing in $p_i$ over the interval $[0,1]$.
\end{restatable}
\begin{proofsketch}
WLOG, fix an arm \(i \in [N]\). Theorem~\ref{thm:fnondecr} follows directly from the second derivative. \(c_1\), \(c_2\), \(c_3\), and \(c_4\) are constants. 
\begin{align}
    \frac{df_i}{dp_i} &= \frac{c_2c_3 - c_1c_4}{(c_3 + c_4p_i)^2}
\end{align}
By the structural constraints $P_{1,1}^0 < P_{1,1}^1$ and $P_{0,1}^0 < P_{0,1}^1$, the numerator is positive; the denominator is always positive.
\end{proofsketch}

By Theorem~\ref{thm:fcurvature}, the arms can be partitioned into two disjoint sets: \(\mathcal{X}=\{i \mid f_i \text{ is concave}\}\) and
\(\mathcal{Y}=\{i \mid f_i \text{ is strictly convex}\}\).
\begin{enumerate}[start=1,label={(\bfseries P\arabic*)}]
    \item maximize $\sum_{i \in \mathcal{X}} f_i(p_i)$ subject to: $p_i \in [\ell, \uuu]$ for all $i \in \mathcal{X}$, and $\sum_{i \in \mathcal{X}} p_i = z$
    \item maximize $\sum_{i \in \mathcal{Y}} f_i(p_i)$ subject to: $p_i \in [\ell, \uuu]$ for all $i \in \mathcal{Y}$, and $\sum_{i \in \mathcal{Y}} p_i = k - z$
\end{enumerate}
Then, \(\pi_{PF}\) is the union of the solutions to \textbf{P1} and \textbf{P2} at the optimal grid search value \(z^* = \argmax_z \sum_{i \in \mathcal{X}} f_i(p_i) + \sum_{i \in \mathcal{Y}} f_i(p_i)\). Algorithm~\ref{alg:probfair} provides pseudocode.

\begin{algorithm}[!h]
\caption{\textsc{ProbFair}}
\label{alg:probfair}
{
\begin{algorithmic}[1]
\Procedure{\textsc{ProbFair}}{$[N], k, \epsilon, \ell, \uuu$}
\State $\mathcal{X} \gets \{i \ | f_i \text{ is \textbf{concave} } \text{in all of } p_i \in [0,1]\}$
\State $\mathcal{Y} \gets \{i \ | f_i \text{ is \textbf{strictly convex} } \text{ in all of } p_i \in [0,1]\}$ 
\State $\texttt{grid\_search\_vals} \gets \{\epsilon j \ | \ j \in [0, k - \ell|\mathcal{Y}|] \}$
% \State $\mathcal{V} \gets \emptyset$
\For{$z \in \texttt{grid\_search\_vals}$}
\State $\pi_{\mathcal{X},z} \gets \Call{SolveP1}{\mathcal{X}, k, z, \ell, \uuu}$
\State $\pi_{\mathcal{Y},z} \gets \Call{SolveP2}{\mathcal{Y}, k, z, \ell, \uuu}$
\State $\mathcal{V}_z \gets \sum_{p_i \in \pi_{\mathcal{X},z}} f_i(p_i) + \sum_{p_i \in \pi_{\mathcal{Y},z}} f_i(p_i)$
\EndFor{}
\State $z^* \gets \argmax_z \mathcal{V}_z$
\State $\pi_{PF} \gets \pi_{\mathcal{X},z^*} \cup \pi_{\mathcal{Y},z^*}$ 
\State\Return $\pi_{PF}$
\EndProcedure
\end{algorithmic}}
\end{algorithm}

\textbf{P1} is a concave-maximization problem that can be solved efficiently via gradient descent. The computational complexity is \(O\!\left(\frac{\lvert \mathcal{X}\rvert}{\varepsilon^2}\right)\)~\citep{nesterov2018lectures}. To solve \textbf{P2}, we begin by introducing a lemma that we prove in Appendix~\ref{app:appProbFair}:
\begin{restatable}[]{lem}{Yoptonboundary}\label{lem:P2_opt_on_boundary}
\textbf{P2} has an optimal solution in which $p_i \in (\ell, \uuu)$ for at most one $i \in \mathcal{Y}$.
\end{restatable}

\begin{proofsketch}
Suppose for contradiction there exists some optimal solution $\vec{p}$ with distinct indices $i, j \in \mathcal{Y}$ such that $p_i, p_j \in (\ell, u)$. Let us compare $\vec{p}$ with a perturbed solution $p_i \coloneqq p_i + \epsilon$ and $p_j \coloneqq p_j - \epsilon$. Using a Taylor series expansion, the change in objective must be $(\epsilon^2/2) \cdot (f_i\dprime(p_i) + f_j\dprime(p_j)) + O(\epsilon^3)$. Since $f_i$ and $f_j$ are strictly convex, $f_i\dprime(p_i) + f_j\dprime(p_j) > 0$. Thus, the objective increases regardless of the sign of (tiny) $\epsilon$, a contradiction.
\end{proofsketch}

Given this structure, an optimal solution \(\{p_i^*\mid i \in \mathcal{Y}\}\) will set some number of arms \(\gamma\in \mathbb{Z}^+\) to \(\ell\), at most one arm to \(p' \in (\ell, \uuu]\), and the remaining \(\lvert \mathcal{Y}\rvert - \gamma-1\) arms to \(\uuu\). We represent these subsets by $\mathcal{Y}_1, \mathcal{Y}_2$, and $\mathcal{Y}_3$, respectively. 
Let \(\gamma = \left\lfloor \frac{\lvert \mathcal{Y}\rvert u-(k-z) }{\uuu-\ell} \right\rfloor\), and $p'=k-z-|\mathcal{Y}_1|\ell - |\mathcal{Y}_3|\uuu~\in~(\ell, \uuu]$. Intuitively, when the remaining budget $k-z$ allows us to set all arms in \(\mathcal{Y}\) to $\uuu$, $\gamma=|\mathcal{Y}_1|=0$. Conversely, when there is only enough budget left to satisfy the fairness constraint for arms in $\mathcal{Y}$, $\gamma=|\mathcal{Y}_1|=|\mathcal{Y}|$. With the cardinality of each subset thus established, per Theorem~\ref{thm:solveP2} (see below), we use Algorithm~\ref{alg:solveP2} to optimally partition the arms in $\mathcal{Y}$.

\begin{algorithm}[H]
\caption{\textsc{SolveP2} \\ Note: all sorts are ascending; arrays are zero-indexed.}
\label{alg:solveP2}
\begin{algorithmic}[1] 
\Procedure{\textsc{SolveP2}}{$\mathcal{Y} \subseteq N, k, z, \ell, \uuu$}
\State $\gamma \gets \left\lfloor\frac{|\mathcal{Y}|\uuu -(k-z)}{\uuu - \ell}\right\rfloor$
\State $p^\prime \gets k - z - \gamma\ell - (|\mathcal{Y}|-1-\gamma)\uuu$
\If{$|\mathcal{Y}| - \gamma - 1 > 0$}
\State $\Delta_{\mathcal{Y}} = \texttt{sort}([f_i(\uuu) - f_i(\ell) \ \forall i \in \mathcal{Y}])$
\State $\mathcal{Y}_{3} \gets \{(\Delta_\mathcal{Y})\left[(|\mathcal{Y}|- \gamma-1): \right]\}$ 
\Else{ $\mathcal{Y}_{3} \gets \emptyset $}
\EndIf
\State $\Delta_{\mathcal{Y} \setminus \mathcal{Y}_3} = \texttt{sort}([f_i(p^\prime) - f_i(\ell) \ \forall i \in \mathcal{Y} \setminus \mathcal{Y}_3])$
\State $\mathcal{Y}_1 \gets \{(\Delta_{\mathcal{Y} \setminus \mathcal{Y}_3})\left[:\gamma \right]\}$
\State $\mathcal{Y}_{2} \gets  \{(\Delta_{\mathcal{Y} \setminus \mathcal{Y}_3})\left[\gamma \right]\}$
\State $\pi_{\mathcal{Y}} := i \mapsto
% \State $\pi_{\mathcal{Y}} := i \mapsto \left.
\left(\ell | i \in \mathcal{Y}_1 \right) \lor \left(p^\prime | i \in \mathcal{Y}_2 \right) \lor \left(\uuu | i \in \mathcal{Y}_3\right)$
%   \begin{cases}
%     \ell, & \text{for } i \in \mathcal{Y}_1 \\ 
%     p^\prime, & \text{for } i \in \mathcal{Y}_2 \\
%     \uuu, & \text{for } i \in \mathcal{Y}_3 \\
%   \end{cases}
%   \right\}$
\State\Return $\pi_\mathcal{Y}$
\EndProcedure
\end{algorithmic}
\end{algorithm}
\setlength{\textfloatsep}{0.2cm}
\setlength{\floatsep}{0.2cm}

\begin{restatable}[]{theorem}{solveconvex}\label{thm:solveP2} 
Alg.~\ref{alg:solveP2} yields the mapping from arms in $\mathcal{Y}$ to subsets in $\{\mathcal{Y}_1,\mathcal{Y}_2,\mathcal{Y}_3\}$ which maximizes $\sum_{i \in \mathcal{Y}} f_i(p_i)$ s.t. \(\sum_{i\in \mathcal{Y}} p_i = k-z\). (See Appendix~\ref{app:appProbFair} for the complete proof).
\end{restatable}

\begin{proofsketch}
By Lemma \ref{lem:P2_opt_on_boundary}, there exists \emph{at most one arm} with optimal value \(p_i^*\in (\ell, u)\). By Lemma \ref{lem:P2_p_prime_unique}, \(\gamma \coloneqq \lvert \mathcal{Y}_1\rvert=\left\lfloor \frac{|\mathcal{Y}|\uuu - (k-z)}{\uuu - \ell} \right\rfloor\) and $p'=k-z-|\mathcal{Y}_1|\ell - |\mathcal{Y}_3|\uuu~\in~(\ell, \uuu]$. Then, we can rewrite Equation~\ref{eqn:probfair_optimization} as an optimization problem over set assignment:
\begin{align}
    &\argmax_{\left\{\mathcal{Y}_1, \mathcal{Y}_2, \mathcal{Y}_3\right\}} ~ \sum_{i\in \mathcal{Y}_1} f_i(\ell) + f_{j}(p')+ \sum_{i'' \in \mathcal{Y}_3} f_{i''}(u) \nonumber \\
    &\textrm{s.t.}~\lvert \mathcal{Y}_1 \rvert = \gamma,~\mathcal{Y}_2 = \{j\}, \nonumber \quad
    % \quad \bigcap_{x=1}^{3} \mathcal{Y}_{x} = \emptyset  ,~\textrm{and}~\bigcup_{x=1}^3 \mathcal{Y}_{x} = \mathcal{Y}
    \bigcap_{x=1}^{3} \mathcal{Y}_{x} = \emptyset  ,~\textrm{and}~\bigcup_{x=1}^3 \mathcal{Y}_{x} = \mathcal{Y} \label{eqn:p2_reorg}.
\end{align}
By algebraic manipulation, assigning the \(\lvert \mathcal{Y}\rvert -\gamma -1\) arms with maximal values of \(f_i(u) - f_i(\ell)\) to \(\mathcal{Y}_3\) produces a maximal solution. Similarly, we assign \(j \in \mathcal{Y}_2\) if \(f_j(p') - f_j(\ell)\) is maximal among the remaining arms. By definition, \(\mathcal{Y}_1^* = \mathcal{Y}\setminus \left(\mathcal{Y}_2 \bigcup\mathcal{Y}_3^*\right)\), which completes the proof.
\end{proofsketch}

\begin{restatable}[]{cor}{complexityconvex}\label{corr:P2_complexity}
Alg. 2 has time complexity $O(|\mathcal{Y}| \log |\mathcal{Y}|)$.
\end{restatable}

With our solutions to \textbf{P1} and \textbf{P2} so defined, the cost of finding our 
%reward-maximizing, fairness and budget-constraint satisfying 
probabilistic policy in this way is 
\(O\!\left(\frac{k-\ell\lvert \mathcal{Y}\rvert}{\varepsilon}\left(\frac{\lvert \mathcal{X}\rvert}{\varepsilon^2}+\lvert \mathcal{Y}\rvert \log\lvert \mathcal{Y}\rvert\right)\right)\), which is at worst \(O\!\left(\frac{kN}{\epsilon^3}\right)\) when all \(N\) arms are in \(\mathcal{X}\). 

\subsection{Sampling Approach}
\label{sec:sampling}
For problem instances with feasible solutions, Algorithm~\ref{alg:probfair} returns $\pi_{PF}$, a mapping from the set of arms to a set of stationary probability distributions over actions, such that for each arm $i$, the probability of receiving a pull at any given timestep is in $[\ell, \uuu]$. By virtue of the fact that $\ell > 0$, this policy guarantees probabilistic fairness constraint satisfaction for all arms. We use a linear-time algorithm introduced by \citet{srinivasan2001distributions} and detailed in Appendix~\ref{sec:AppProbFairSampling} to sample from $\pi_{PF}$ at each timestep, such that the following properties hold: (1) with probability one, we satisfy the budget constraint by pulling exactly $k$ arms; and (2) any given arm $i$ is pulled with probability $p_i$. Formally, each time we draw a vector of binary random variables $(X_1, X_2 \dots X_N)$ from the distribution $\pi_{PF}$,  $\text{Pr}\left[|i: X_i =1| = k\right] = 1$ and $\forall i, \text{Pr}[X_i = 1] = p_i$.

\section{Experimental Evaluation}
\label{sec:experimentalEvaluation}
In this section, we empirically demonstrate that \textsc{ProbFair} enforces the probabilistic fairness constraint introduced in Section~\ref{sec:rmabModel} with minimal loss in total expected reward, relative to fairness-aware alternatives. We begin by identifying our comparison policies, evaluation metrics, and datasets. We then present results from three experiments: (1) \textsc{ProbFair} versus fairness-inducing alternative policies, holding the cohort fixed and considering fairness-aligned sets of hyperparameters; (2) \textsc{ProbFair} evaluated on a breadth of cohorts representing different types of patient populations; and (3) \textsc{ProbFair} when fairness is \emph{not} enforced (i.e., $\ell = 0$), to examine the cost of state agnosticism.\footnote{Code to reproduce our empirical results is provided at \url{https://github.com/crherlihy/prob_fair_rmab}.}  

\subsection{Experimental Setup}
\label{sec:experimentalDesign}
\par \textbf{Policies}: In our experiments, we compare \textsc{ProbFair} against a subset of the following baseline\baseline{} and fairness-\{inducing\fairInduce{}, guaranteeing\fairGuarantee{}, and agnostic\fairAgnostic{}\} policies:
\begin{table}[!h]
\resizebox{\columnwidth}{!}{%
\begin{tabular}{cl}
\hline
\multicolumn{1}{|c|}{\textsc{Random}\baseline{} }      & \multicolumn{1}{l|}{Select $k$ arms uniformly at random at each $t$.}           \\ \hline
\multicolumn{1}{|c|}{\textsc{Round-Robin}\baseline{},\fairGuarantee{}} & \multicolumn{1}{l|}{Select $k$ arms at each $t$ in fixed, sequential order.} \\ \hline
\multicolumn{1}{|c|}{\begin{tabular}[c]{@{}c@{}} \textsc{TW-based} \\ \textsc{heuristics}\fairGuarantee{} \end{tabular}} &
  \multicolumn{1}{l|}{\begin{tabular}[c]{@{}l@{}}Select top-$k$ arms based on Whittle index values.\\ Available arms vary based on time-indexed \\ fairness constraint  satisfaction~\citep{prins2020incorporating}. \end{tabular}} \\ \hline
\multicolumn{1}{|c|}{\begin{tabular}[c]{@{}c@{}} \textsc{Risk-Aware} \\ \textsc{TW (RA-TW)}\fairInduce{} \end{tabular}} &
  \multicolumn{1}{l|}{\begin{tabular}[c]{@{}l@{}}Select top-$k$ arms based on Whittle index values.\\ Incentivizes fairness via concave reward function\\ \citep{mate2021risk-aware}. \end{tabular}} \\ \hline
\multicolumn{1}{|c|}{\begin{tabular}[c]{@{}c@{}}\textsc{Threshold} \\ \textsc{Whittle (TW)}\fairAgnostic{ }\end{tabular}} &
  \multicolumn{1}{l|}{\begin{tabular}[c]{@{}l@{}}Select top-$k$ arms based on Whittle index values\\ \citep{whittle1988restless, mate2020collapsing}.\end{tabular}} \\ \hline
\end{tabular}
}
\end{table}

We specifically consider three \textbf{\textsc{Threshold Whittle}}-based heuristics: \textsc{H}$_{\textsc{First}}$, \textsc{H}$_{\textsc{Last}}$, and \textsc{H}$_{\textsc{Rand}}$. These heuristics partition the $k$ pulls available at each timestep into (un)constrained subsets, where a pull is \emph{constrained} if it is executed to satisfy a time-indexed fairness constraint. During constrained pulls, only arms that have not yet been pulled the required number of times within a $\nu$-length interval are available; other arms are excluded from consideration, unless \emph{all} arms have already satisfied their constraints. \textsc{H}$_{\textsc{First}}$, \textsc{H}$_{\textsc{Last}}$, and \textsc{H}$_{\textsc{Rand}}$ position constrained pulls at the beginning, end, or randomly within each interval of length \(\nu\), respectively. Appendix~\ref{sec:AppHeuristics} provides pseudocode.

\par \textbf{Objective}: 
In all experiments, we assign equal value to the adherence of a given arm over time. Thus, we set our objective to reward occupancy in the ``good'' state: a simple local reward \(r_t(s_t^i) \coloneqq s_t^i \in \{0,1\}\) and undiscounted cumulative reward function,  \({R}(r(s)) \coloneqq \sum_{i \in [N]} \sum_{t\in[T]} r(s^i_t)\). 

\par \textbf{Evaluation metrics}: We are interested in comparing policies along two dimensions: reward maximization and fairness (i.e., with respect to the distribution of algorithmic attention). To this end, we rely on two performance metrics: (a) intervention benefit and (b) earth mover's distance. 

\emph{Intervention benefit (IB)} is the total expected reward of an algorithm, normalized between the reward obtained with no interventions (0\% intervention benefit) and the asymptotically optimal but fairness-agnostic \textsc{Threshold Whittle} algorithm (100\%)~\citep{mate2020collapsing}. Formally,
\begin{equation}\label{eq:ib}
\textsc{IB}_{\text{NoAct}, \text{TW}}(\text{ALG}) \coloneqq \frac{\mathbb{E}_\text{ALG}[R_\text{ALG}(\cdot)] - \mathbb{E}_\text{NoAct}[R(\cdot)]}{\mathbb{E}_\text{TW}[R(\cdot)] - \mathbb{E}_\text{NoAct}[R(\cdot)]}
\end{equation}

Per Lemma~\ref{thm:IBcorrPoF} (App.~\ref{sec:AppExpMetrics}), the \textit{price of fairness (PoF)} metric~\citep{bertsimas2011price} is inversely proportional to intervention benefit. We thus report IB.

\emph{Earth mover's distance (EMD)} is a metric that allows us to compute the minimum cost required to transform one probability distribution into another~\citep{rubner2000earth}. We use it to compare algorithms with respect to fairness---i.e., how evenly a set of pulls are allocated among arms. (Other metrics that may measure individual distributive fairness are discussed in Appendix~\ref{sec:AppExpMetrics}.)

For each algorithm, we consider a discrete distribution \(F\) of observed pull counts, where each bucket, $j \in \{0 \dots T\}$, corresponds to a feasible number of total pulls that an arm could receive, and $F[j] \in \{0 \dots N\}$ corresponds to the number of arms whose observed pull count is equal to $j$. For example, $F[0]$ corresponds to the quantity of arms never pulled, and $F[T]$ corresponds to the quantity of arms pulled at every timestep. Each algorithm produces $kT$ total pulls, so the distributions have the same total mass.

We use \textsc{Round-Robin} as a fair reference algorithm since it distributes pulls evenly among arms. We then compute the minimum cost required to transform each algorithm's distribution, $F_{\text{ALG}}$, into that of \textsc{Round-Robin}'s, $F_{\text{RR}}$. 

For our application this is equivalent to:
\begin{equation}\label{eqn:EMD}
    \textsc{EMD}_{\text{RR}}(\text{ALG}) \coloneqq \left\lvert \sum_{h=0}^T \sum_{j=0}^h F_{\text{ALG}}[j] - F_{\text{RR}}[j]  \right\rvert
\end{equation}

Unless otherwise noted, we normalize EMD such that the maximum distance we encounter, that of TW, is one:
\begin{equation}
    \frac{\textsc{EMD}_{\text{RR}}(\text{ALG}) - \textsc{EMD}_{\text{RR}}(\text{RR})}{\textsc{EMD}_{\text{RR}}(\text{TW}) - \textsc{EMD}_{\text{RR}}(\text{RR})} = \frac{\textsc{EMD}_{\text{RR}}(\text{ALG})}{\textsc{EMD}_{\text{RR}}(\text{TW})}
\end{equation}

\par \textbf{Datasets}: We evaluate performance on two datasets: (a) a realistic patient adherence behavior model and (b) a general set of randomly generated synthetic transition matrices.

\emph{CPAP Adherence}. Obstructive sleep apnea (OSA) is a common condition that causes interrupted breathing during sleep~\citep{punjabi2008epidemiology}; when used throughout the entirety of sleep, continuous positive airway pressure therapy (CPAP) eliminates nearly 100\% of obstructive apneas for the majority of treated patients~\citep{sawyer2011systematic}. However, poor adherence behavior in using CPAP reduces its beneficial outcomes. CPAP non-adherence affects an estimated 30-40\% of patients~\citep{rotenberg2016trends}.

We derive the CPAP dataset that we use in our experiments from the work of \citet{kang2013markov, kang2016modelling}, who model the dynamics and patterns of patient adherence behavior as a basis for designing effective and economical interventions. In particular, we adapt their Markov model of CPAP adherence behavior (a three-state system based on hours of nightly CPAP usage) to a two-state system using the clinical standard for adherence--at least four hours of CPAP machine usage per night~\citep{sawyer2011systematic}. \citet{kang2013markov} find, via expectation-maximization on CPAP usage patterns, that patients can be divided into two groups based on this clinical standard. Though patients in the first cluster occasionally miss a night, these patients utilize a CPAP machine for more than four hours every night without assistance, while patients in the second cluster do not. We refer to the latter cluster as the \textit{non-adherent} cohort in our analysis.

\citet{kang2016modelling} consider many intervention effects. We specifically consider an intervention effect, \(\alpha_{\text{interv}}=1.1\), that broadly characterizes supportive interventions such as telemonitoring and phone support, which are associated with a moderate 0.70 hours (95\% CI \(\pm~ 0.35\)) increase in device usage per night~\citep{askland2020educational}. We add random \(\sigma=1\) logistic noise to the transition matrices so that there is some variance in individual arm dynamics. To prevent overlap with the general cohort we consider for contrast, added noise can only \textit{hinder} the probability of adherence in the non-adherent cohort.

\emph{Synthetic}. In addition, we construct a synthetic dataset of randomly generated arms such that the structural constraints outlined in Section~\ref{sec:rmabModel} are preserved. We conjecture that forward (reverse) threshold-optimal arms are a subset of concave (strictly convex) arms (see Appendix~\ref{sec:appSynthdata}).

\subsection{\textsc{ProbFair} vs. Fairness-aware Alternatives} 
\label{sec:pfVsFair}
Here we compare \textsc{ProbFair} to policies which either \emph{induce} or \emph{guarantee} fairness. The former includes \textsc{Risk-aware Whittle (RA-TW)}, which incentivizes fairness via concave reward \(r(b)\)~\citep{mate2021risk-aware}. We use the authors' suggested reward function \(r(b) = -e^{\lambda(1-b)}\), \(\lambda =20\). This imposes a large negative utility on lower belief values, which motivates preemptive intervention. However, \textsc{RA-TW} does not \emph{guarantee} time-indexed or probabilistic fairness for individual arms. The latter includes \textsc{Round-Robin} and the \textsc{First}, \textsc{Last}, and \textsc{Random} heuristics, which guarantee time-indexed fairness but do \emph{not} provide any optimality guarantees. 

\newcolumntype{V}{!{\vrule width 1pt}}
\begin{table}[!h]
\centering
\resizebox{\columnwidth}{!}{%
\begin{tabular}{|c|l V l|l|}
	\specialrule{1pt}{1pt}{1pt}
	$\min_i\mathbb{E}[\text{\# pulls}]$& Policy & $\mathbb{E}[\textsc{IB}]$ (\%)& $\mathbb{E}[\textsc{EMD}]$ (\%) \\
	\specialrule{2.5pt}{1pt}{1pt}
	\multirow{4}{*}{\makecell{\textsc{10}\\ $\ell= 0.056$ \\ $\nu = 18$}} &\textsc{PF} \hfill $\ell$  & 88.73 ~$\pm$ 0.26 & 81.78 ~$\pm$ 0.18 \\
	 &\textsc{H}$_{\textsc{First}}$  \hfill $\nu$ & 86.11 ~$\pm$ 0.26 & 71.53 ~$\pm$ 0.13 \\
	 &\textsc{H}$_{\textsc{Last}}$   \hfill $\nu$ & 87.37 ~$\pm$ 0.28  & \textbf{70.48 ~$\pm$ 0.12 } \\	 &\textsc{H}$_{\textsc{Rand}}$   \hfill $\nu$ &\textbf{90.79 ~$\pm$ 0.22} & 74.12 ~$\pm$ 0.15 \\
	\specialrule{1.5pt}{1pt}{1pt}	\multirow{4}{*}{\makecell{\textsc{18}\\ $\ell= 0.1$ \\ $\nu = 10$}} &\textsc{PF} \hfill $\ell$  & 80.80 ~$\pm$ 0.30 & 59.96 ~$\pm$ 0.19 \\
	 &\textsc{H}$_{\textsc{First}}$  \hfill $\nu$ & 76.62 ~$\pm$ 0.30 & 49.54 ~$\pm$ 0.09 \\
	 &\textsc{H}$_{\textsc{Last}}$  \hfill $\nu$ & 77.95 ~$\pm$ 0.30  & \textbf{49.26 ~$\pm$ 0.08 } \\	 &\textsc{H}$_{\textsc{Rand}}$ \hfill $\nu$ &\textbf{81.53 ~$\pm$ 0.30} & 52.73 ~$\pm$ 0.10 \\
	\specialrule{1.5pt}{1pt}{1pt}	\multirow{4}{*}{\makecell{\textsc{30}\\ $\ell= 0.167$ \\ $\nu = 6$}} &\textsc{PF} \hfill $\ell$  &\textbf{66.12 ~$\pm$ 0.35} & 23.61 ~$\pm$ 0.12 \\
	 &\textsc{H}$_{\textsc{First}}$  \hfill $\nu$ & 63.58 ~$\pm$ 0.31  & \textbf{18.98 ~$\pm$ 0.03 } \\	 &\textsc{H}$_{\textsc{Last}}$  \hfill $\nu$ & 64.63 ~$\pm$ 0.34 & 19.47 ~$\pm$ 0.04 \\
	 &\textsc{H}$_{\textsc{Rand}}$  \hfill $\nu$ & 65.21 ~$\pm$ 0.32 & 19.64 ~$\pm$ 0.04 \\
	\specialrule{1.5pt}{1pt}{1pt}	\multirow{2}{*}{comparison} &RA-TW & 85.12 ~$\pm$ 0.42  & \textbf{95.80 ~$\pm$ 0.42 } \\	 &TW &\textbf{100.00 ~$\pm$ 0.00} & 100.00 ~$\pm$ 0.00 \\
	\specialrule{1.5pt}{1pt}{1pt}	\multirow{3}{*}{baseline} &\textsc{Random}  & 50.02 ~$\pm$ 0.35 & 10.08 ~$\pm$ 0.10 \\
	 &\textsc{NoAct}  & 0.00 ~$\pm$ 0.00 & 73.48 ~$\pm$ 0.13 \\
	 &\textsc{RR}  &\textbf{56.96 ~$\pm$ 0.33}  & \textbf{0.00 ~$\pm$ 0.00 } \\	
	 \specialrule{1.5pt}{1pt}{1pt}
	\end{tabular}%
}
\caption{Expected intervention benefit and normalized earth mover's distance by policy and fairness bracket.}
\label{tab:exp1table}
\end{table}

In Table~\ref{tab:exp1table}, we report average results for each policy, along with margins of error for $95\%$ confidence intervals, computed over $100$ simulation seeds for a synthetic cohort of $100$ collapsing arms, with $k=20$ and $T=180$. To facilitate meaningful comparisons between \textsc{ProbFair} and the heuristics, we consider combinations of values for $\ell$ and $\nu$ that produce equivalent, integer-valued lower bounds on the number of pulls any arm can expect to receive---i.e., $\min_i \mathbb{E}[\sum_{t} \mathbbm{1}(a^i_t=1)] = \ell \times T = \frac{T}{\nu}$. 

\vspace{1em}
Key findings from this experiment include:
\begin{itemize}
    \item
    \emph{Fairer} hyperparameter values (i.e, $\ell \uparrow$, $\nu \downarrow$), correspond to decreases in $\mathbb{E}[\textsc{IB}]$ and $\mathbb{E}[\textsc{EMD}]$, reflecting improved individual fairness at the expense of total expected reward.
    \item \textsc{ProbFair} is competitive with respect to $\textsc{RA-TW}$, outperforming on both metrics when $\ell=0.056$, and incurring a slight loss in $\mathbb{E}[\textsc{IB}]$ but improvement in $\mathbb{E}[\textsc{EMD}]$ for $\ell=0.1$.
    \item For each $(\ell, \nu)$ combination, \textsc{ProbFair} performs competitively with respect to the best-performing heuristic (which, like \textsc{TW}, are state-aware, see Section~\ref{sec:expNoFair}).
\end{itemize}

\subsection{\textsc{ProbFair} on a Breadth of Cohorts}\label{sec:expBreadth}
In this section we conduct sensitivity analysis with respect to cohort composition. For each dataset, we identify a transition matrix characteristic that can be modified during the generation process to produce a subset of arms that will exhibit less favorable transition dynamics than their peers. For the synthetic dataset, this characteristic is \emph{strict convexity}. For the CPAP dataset, it is \emph{non-adherence}, a mnemonic coined by \citet{kang2013markov} to characterize a cluster of study participants, and contrast this to a model fit on the general patient population.

For each dataset, we generate ten different cohorts, each of which is characterized by the percentage of unfavorable arms that it contains. We use a seed to control the generation process such that each cohort contains 100 collapsing arms in total. A sliding window of the unfavorable %(favorable) 
arms we can generate with this seed are included %(excluded) 
as we increase the cardinality of the unfavorable subset. For ease of interpretation, we present unnormalized results over 100 simulation seeds with $k=20$ and $T=180$ in Figure~\ref{fig:exp2fig}, and then proceed to summarize normalized performance.
\begin{figure}[!h]
\begin{center}
\resizebox{\columnwidth}{!}{
\subfigure[Synthetic results]{\label{fig:exp2a}\includegraphics[width=\linewidth]{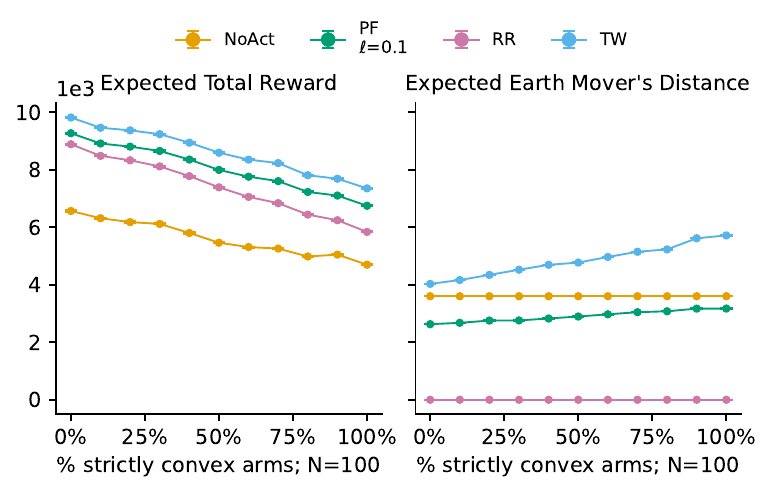}}}
\resizebox{\columnwidth}{!}{
\subfigure[CPAP results]{\label{fig:exp2b}\includegraphics[width=\linewidth]{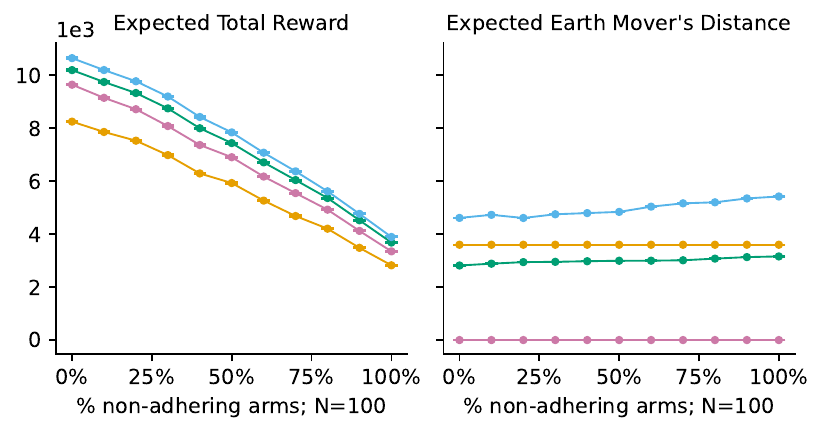}}}
\caption{Expected total reward (left) and unnormalized earth mover's distance (right) on a breadth of cohorts.}
\label{fig:exp2fig}
\end{center}
\end{figure}

\vspace{1em}
Key findings from this experiment include:
\begin{itemize}
    \item Per Figure~\ref{fig:exp2fig}, for each dataset, expected total reward predictably declines for all policies as the percentage of unfavorable arms increases, while unnormalized EMD increases for \textsc{TW} and \textsc{ProbFair}. 
    \begin{itemize}
        \item \emph{Synthetic}: As the proportion of strictly convex arms increases, \textsc{ProbFair}'s allocation of resources tends towards the bimodality of \textsc{TW}. 
        \item \emph{CPAP}: As the proportion of non-adherent arms increases, the level of intervention required to improve trajectories rises, but the budget constraint is static. % Thus, more arms receive lower $p_i$s, leading to rising unnormalized EMD.
    \end{itemize}
    \item For each dataset, \textsc{ProbFair}'s normalized performance remains stable even as cohort composition is varied: 
    \begin{itemize}
        \item \emph{Synthetic}: With respect to IB (EMD), \textsc{ProbFair} achieves an average (over all cohorts) of averages (over 100 simulations per cohort) of $80.69\%~\pm~1.42\%$ ($58.98\%~\pm~1.29\%$). 
        \item \emph{CPAP}: The values for IB (EMD) are: $79.84\%~\pm~0.68\%$ ($59.68\% ~\pm~1.08\%$).
    \end{itemize}
\end{itemize}

\subsection{\textsc{ProbFair}: Price of State Agnosticism} \label{sec:expNoFair}
Here, we investigate the cost associated with \textsc{ProbFair}'s state agnosticism, relative to state-aware \textsc{Threshold Whittle}. To ensure a fair comparison, we set $\ell=0$ and $\uuu=1$, effectively constructing a version of \textsc{ProbFair} in which probabilistic fairness is not enforced. (Recall that \textsc{TW} is fairness-agnostic; in the previous results, we do not expect \textsc{ProbFair} to obtain the same total reward as \textsc{TW}).

Although \textsc{ProbFair} incorporates each arm's structural information (i.e., transition matrices), it produces a set of \emph{stationary} probability distributions over actions from which all discrete actions are subsequently drawn. \textsc{TW}, in contrast, ingests each arm's current state at each timestep, and is thus able to exploit \emph{realized} sequences of state transitions. 

While we thus expect \textsc{ProbFair} to incur some loss in intervention benefit, our results (computed over 100 simulation seeds, with $k=20$, $N=100$, and $T=180$) indicate that this loss is acceptable rather than catastrophic. Relative to \textsc{TW}, $\textsc{ProbFair}_{\ell = 0}$ obtains $97.41\% \pm 0.26$ of $\mathbb{E}[\textsc{IB}]$ and incurs an increase of only $4.56\% \pm 0.19$ with respect to $\mathbb{E}[\textsc{EMD}]$. 

\section{Conclusion and Future Work}
\label{sec:conclusion}
In this paper, we introduce \textsc{ProbFair}, a novel, probabilistically fair algorithm for constrained resource allocation. Our theoretical results prove that this policy is reward-maximizing, subject to the guaranteed satisfaction of both budget and tunable probabilistic fairness constraints. Our empirical results demonstrate that \textsc{ProbFair} preserves utility while providing fairness guarantees.  Promising future directions include: (1) extending \textsc{ProbFair} to address larger state and/or action spaces; and (2) relaxing the requirement for stationarity in the construction of $\pi_{PF}$.

\begin{acks}
Christine Herlihy was supported by the National Institute of Standards and Technology’s (NIST) Professional Research Experience Program (PREP). John Dickerson and Aviva Prins were supported in part by NSF CAREER Award IIS-1846237, NSF D-ISN Award \#2039862, NSF Award CCF-1852352, NIH R01 Award NLM-013039-01, NIST MSE Award \#20126334, DARPA GARD \#HR00112020007, DoD WHS Award \#HQ003420F0035, and a Google Faculty Research award. Aravind Srinivasan was supported in part by NSF awards CCF-1749864 and CCF-1918749, as well as research awards from Adobe, Amazon, and Google. The views and conclusions contained in this publication are those of the authors and should not be interpreted as representing official policies or endorsements of U.S. government or funding agencies. We thank Samuel Dooley, Dr.\ Furong Huang, Naveen Raman, and Daniel Smolyak for helpful input and feedback.
\end{acks}

\bibliographystyle{ACM-Reference-Format}
\balance
\bibliography{main}


\begin{thebibliography}{42}

%%% the Web address.

\ifx \showCODEN    \undefined \def \showCODEN     #1{\unskip}     \fi
\ifx \showDOI      \undefined \def \showDOI       #1{#1}\fi
\ifx \showISBNx    \undefined \def \showISBNx     #1{\unskip}     \fi
\ifx \showISBNxiii \undefined \def \showISBNxiii  #1{\unskip}     \fi
\ifx \showISSN     \undefined \def \showISSN      #1{\unskip}     \fi
\ifx \showLCCN     \undefined \def \showLCCN      #1{\unskip}     \fi
\ifx \shownote     \undefined \def \shownote      #1{#1}          \fi
\ifx \showarticletitle \undefined \def \showarticletitle #1{#1}   \fi
\ifx \showURL      \undefined \def \showURL       {\relax}        \fi
% invisible to TeX
\providecommand\bibfield[2]{#2}
\providecommand\bibinfo[2]{#2}
\providecommand\natexlab[1]{#1}
\providecommand\showeprint[2][]{arXiv:#2}

\bibitem[Askland et~al\mbox{.}(2020)]%
        {askland2020educational}
\bibfield{author}{\bibinfo{person}{Kathleen Askland}, \bibinfo{person}{Lauren
  Wright}, \bibinfo{person}{Dariusz~R Wozniak}, \bibinfo{person}{Talia
  Emmanuel}, \bibinfo{person}{Jessica Caston}, {and} \bibinfo{person}{Ian
  Smith}.} \bibinfo{year}{2020}\natexlab{}.
\newblock \showarticletitle{{Educational, Supportive and Behavioural
  Interventions to Improve Usage of Continuous Positive Airway Pressure
  Machines in Adults with Obstructive Sleep Apnoea}}.
\newblock \bibinfo{journal}{\emph{Cochrane Database of Systematic Reviews}}
  \bibinfo{number}{4} (\bibinfo{year}{2020}).
\newblock


\bibitem[Ayer et~al\mbox{.}(2019)]%
        {ayer2019prioritizing}
\bibfield{author}{\bibinfo{person}{Turgay Ayer}, \bibinfo{person}{Can Zhang},
  \bibinfo{person}{Anthony Bonifonte}, \bibinfo{person}{Anne~C Spaulding},
  {and} \bibinfo{person}{Jagpreet Chhatwal}.} \bibinfo{year}{2019}\natexlab{}.
\newblock \showarticletitle{{Prioritizing Hepatitis C Treatment in US
  Prisons}}.
\newblock \bibinfo{journal}{\emph{Operations Research}} \bibinfo{volume}{67},
  \bibinfo{number}{3} (\bibinfo{year}{2019}), \bibinfo{pages}{853--873}.
\newblock


\bibitem[Bentham(1781)]%
        {bentham1781introduction}
\bibfield{author}{\bibinfo{person}{Jeremy Bentham}.}
  \bibinfo{year}{1781}\natexlab{}.
\newblock \bibinfo{booktitle}{\emph{{An Introduction to the Principles of
  Morals and Legislation}}}.
\newblock \bibinfo{type}{{T}echnical {R}eport}. \bibinfo{institution}{McMaster
  University Archive for the History of Economic Thought}.
\newblock


\bibitem[Bertsimas et~al\mbox{.}(2011)]%
        {bertsimas2011price}
\bibfield{author}{\bibinfo{person}{Dimitris Bertsimas},
  \bibinfo{person}{Vivek~F Farias}, {and} \bibinfo{person}{Nikolaos
  Trichakis}.} \bibinfo{year}{2011}\natexlab{}.
\newblock \showarticletitle{{The Price of Fairness}}.
\newblock \bibinfo{journal}{\emph{Operations research}} \bibinfo{volume}{59},
  \bibinfo{number}{1} (\bibinfo{year}{2011}), \bibinfo{pages}{17--31}.
\newblock


\bibitem[Biswas et~al\mbox{.}(2023)]%
        {biswasfairness}
\bibfield{author}{\bibinfo{person}{Arpita Biswas}, \bibinfo{person}{Jackson~A
  Killian}, \bibinfo{person}{Paula~Rodriguez Diaz}, \bibinfo{person}{Susobhan
  Ghosh}, {and} \bibinfo{person}{Milind Tambe}.}
  \bibinfo{year}{2023}\natexlab{}.
\newblock \showarticletitle{Fairness for Workers Who Pull the Arms: An Index
  Based Policy for Allocation of Restless Bandit Tasks}. In
  \bibinfo{booktitle}{\emph{22nd International Conference on Autonomous Agents
  and Multiagent Systems (AAMAS)}}. \bibinfo{address}{London, UK}.
\newblock


\bibitem[Butler et~al\mbox{.}(2020)]%
        {butler2020us}
\bibfield{author}{\bibinfo{person}{Catherine~R. Butler}, \bibinfo{person}{Susan
  P.~Y. Wong}, \bibinfo{person}{Aaron~G. Wightman}, {and}
  \bibinfo{person}{Ann~M. O’Hare}.} \bibinfo{year}{2020}\natexlab{}.
\newblock \showarticletitle{{US Clinicians’ Experiences and Perspectives on
  Resource Limitation and Patient Care During the COVID-19 Pandemic}}.
\newblock \bibinfo{journal}{\emph{JAMA Network Open}} \bibinfo{volume}{3},
  \bibinfo{number}{11} (\bibinfo{date}{11} \bibinfo{year}{2020}),
  \bibinfo{pages}{e2027315--e2027315}.
\newblock
\showISSN{2574-3805}
\urldef\tempurl%
\url{https://doi.org/10.1001/jamanetworkopen.2020.27315}
\showDOI{\tempurl}


\bibitem[Chen et~al\mbox{.}(2020)]%
        {chenFairContext}
\bibfield{author}{\bibinfo{person}{Yifang Chen}, \bibinfo{person}{Alex
  Cuellar}, \bibinfo{person}{Haipeng Luo}, \bibinfo{person}{Jignesh Modi},
  \bibinfo{person}{Heramb Nemlekar}, {and} \bibinfo{person}{Stefanos
  Nikolaidis}.} \bibinfo{year}{2020}\natexlab{}.
\newblock \showarticletitle{{Fair Contextual Multi-Armed Bandits: Theory and
  Experiments}} \emph{(\bibinfo{series}{Proceedings of Machine Learning
  Research}, Vol.~\bibinfo{volume}{124})},
  \bibfield{editor}{\bibinfo{person}{Jonas Peters} {and} \bibinfo{person}{David
  Sontag}} (Eds.). \bibinfo{publisher}{PMLR}, \bibinfo{address}{Virtual},
  \bibinfo{pages}{181--190}.
\newblock
\urldef\tempurl%
\url{http://proceedings.mlr.press/v124/chen20a.html}
\showURL{%
\tempurl}


\bibitem[De-Arteaga et~al\mbox{.}(2020)]%
        {de2020case}
\bibfield{author}{\bibinfo{person}{Maria De-Arteaga}, \bibinfo{person}{Riccardo
  Fogliato}, {and} \bibinfo{person}{Alexandra Chouldechova}.}
  \bibinfo{year}{2020}\natexlab{}.
\newblock \showarticletitle{{A Case for Humans-in-the-Loop: Decisions in the
  Presence of Erroneous Algorithmic Scores}}. In
  \bibinfo{booktitle}{\emph{Proceedings of the 2020 CHI Conference on Human
  Factors in Computing Systems}}. \bibinfo{pages}{1--12}.
\newblock


\bibitem[Dietvorst et~al\mbox{.}(2015)]%
        {dietvorst2015algorithm}
\bibfield{author}{\bibinfo{person}{Berkeley~J Dietvorst},
  \bibinfo{person}{Joseph~P Simmons}, {and} \bibinfo{person}{Cade Massey}.}
  \bibinfo{year}{2015}\natexlab{}.
\newblock \showarticletitle{{Algorithm aversion: People erroneously avoid
  algorithms after seeing them err}}.
\newblock \bibinfo{journal}{\emph{Journal of Experimental Psychology: General}}
  \bibinfo{volume}{144}, \bibinfo{number}{1} (\bibinfo{year}{2015}),
  \bibinfo{pages}{114}.
\newblock


\bibitem[Dwork et~al\mbox{.}(2011)]%
        {dwork2011fairness}
\bibfield{author}{\bibinfo{person}{Cynthia Dwork}, \bibinfo{person}{Moritz
  Hardt}, \bibinfo{person}{Toniann Pitassi}, \bibinfo{person}{Omer Reingold},
  {and} \bibinfo{person}{Richard~S. Zemel}.} \bibinfo{year}{2011}\natexlab{}.
\newblock \showarticletitle{Fairness Through Awareness}.
\newblock \bibinfo{journal}{\emph{CoRR}}  \bibinfo{volume}{abs/1104.3913}
  (\bibinfo{year}{2011}).
\newblock
\showeprint[arXiv]{1104.3913}
\urldef\tempurl%
\url{http://arxiv.org/abs/1104.3913}
\showURL{%
\tempurl}


\bibitem[Ghouila-Houri(1962)]%
        {ghouila1962caracterisation}
\bibfield{author}{\bibinfo{person}{Alain Ghouila-Houri}.}
  \bibinfo{year}{1962}\natexlab{}.
\newblock \showarticletitle{{Caract{\'e}risation des matrices totalement
  unimodulaires}}.
\newblock \bibinfo{journal}{\emph{Comptes Redus Hebdomadaires des S{\'e}ances
  de l'Acad{\'e}mie des Sciences (Paris)}}  \bibinfo{volume}{254}
  (\bibinfo{year}{1962}), \bibinfo{pages}{1192--1194}.
\newblock


\bibitem[Hirschman(1980)]%
        {hirschman1980national}
\bibfield{author}{\bibinfo{person}{Albert~O Hirschman}.}
  \bibinfo{year}{1980}\natexlab{}.
\newblock \bibinfo{booktitle}{\emph{{National Power and the Structure of
  Foreign Trade}}}. Vol.~\bibinfo{volume}{105}.
\newblock \bibinfo{publisher}{Univ of California Press}.
\newblock


\bibitem[Hou et~al\mbox{.}(2009)]%
        {hou2009theory}
\bibfield{author}{\bibinfo{person}{I-H Hou}, \bibinfo{person}{Vivek Borkar},
  {and} \bibinfo{person}{PR Kumar}.} \bibinfo{year}{2009}\natexlab{}.
\newblock \bibinfo{booktitle}{\emph{A theory of QoS for wireless}}.
\newblock \bibinfo{publisher}{IEEE}.
\newblock


\bibitem[Joseph et~al\mbox{.}(2016)]%
        {joseph2016fairness}
\bibfield{author}{\bibinfo{person}{Matthew Joseph}, \bibinfo{person}{Michael
  Kearns}, \bibinfo{person}{Jamie~H Morgenstern}, {and} \bibinfo{person}{Aaron
  Roth}.} \bibinfo{year}{2016}\natexlab{}.
\newblock \showarticletitle{{Fairness in Learning: Classic and Contextual
  Bandits}}.
\newblock In \bibinfo{booktitle}{\emph{Advances in Neural Information
  Processing Systems 29}}, \bibfield{editor}{\bibinfo{person}{D.~D. Lee},
  \bibinfo{person}{M.~Sugiyama}, \bibinfo{person}{U.~V. Luxburg},
  \bibinfo{person}{I.~Guyon}, {and} \bibinfo{person}{R.~Garnett}} (Eds.).
  \bibinfo{publisher}{Curran Associates, Inc.}, \bibinfo{pages}{325--333}.
\newblock
\urldef\tempurl%
\url{http://papers.nips.cc/paper/6355-fairness-in-learning-classic-and-contextual-bandits.pdf}
\showURL{%
\tempurl}


\bibitem[Jung and Tewari(2019)]%
        {jung2019regret}
\bibfield{author}{\bibinfo{person}{Young~Hun Jung} {and} \bibinfo{person}{Ambuj
  Tewari}.} \bibinfo{year}{2019}\natexlab{}.
\newblock \showarticletitle{{Regret Bounds for Thompson Sampling in Episodic
  Restless Bandit Problems}}.
\newblock In \bibinfo{booktitle}{\emph{Advances in Neural Information
  Processing Systems 32}}, \bibfield{editor}{\bibinfo{person}{H.~Wallach},
  \bibinfo{person}{H.~Larochelle}, \bibinfo{person}{A.~Beygelzimer},
  \bibinfo{person}{F.~d\textquotesingle Alch\'{e}-Buc},
  \bibinfo{person}{E.~Fox}, {and} \bibinfo{person}{R.~Garnett}} (Eds.).
  \bibinfo{publisher}{Curran Associates, Inc.}, \bibinfo{pages}{9007--9016}.
\newblock
\urldef\tempurl%
\url{http://papers.nips.cc/paper/9102-regret-bounds-for-thompson-sampling-in-episodic-restless-bandit-problems.pdf}
\showURL{%
\tempurl}


\bibitem[Kang et~al\mbox{.}(2013)]%
        {kang2013markov}
\bibfield{author}{\bibinfo{person}{Yuncheol Kang}, \bibinfo{person}{Vittaldas~V
  Prabhu}, \bibinfo{person}{Amy~M Sawyer}, {and} \bibinfo{person}{Paul~M
  Griffin}.} \bibinfo{year}{2013}\natexlab{}.
\newblock \showarticletitle{{Markov models for treatment adherence in
  obstructive sleep apnea}}.
\newblock \bibinfo{journal}{\emph{Age}}  \bibinfo{volume}{49}
  (\bibinfo{year}{2013}), \bibinfo{pages}{11--6}.
\newblock


\bibitem[Kang et~al\mbox{.}(2016)]%
        {kang2016modelling}
\bibfield{author}{\bibinfo{person}{Yuncheol Kang}, \bibinfo{person}{Amy~M
  Sawyer}, \bibinfo{person}{Paul~M Griffin}, {and} \bibinfo{person}{Vittaldas~V
  Prabhu}.} \bibinfo{year}{2016}\natexlab{}.
\newblock \showarticletitle{{Modelling Adherence Behaviour for the Treatment of
  Obstructive Sleep Apnoea}}.
\newblock \bibinfo{journal}{\emph{European journal of operational research}}
  \bibinfo{volume}{249}, \bibinfo{number}{3} (\bibinfo{year}{2016}),
  \bibinfo{pages}{1005--1013}.
\newblock


\bibitem[Kelly et~al\mbox{.}(2019)]%
        {kelly2019key}
\bibfield{author}{\bibinfo{person}{Christopher~J Kelly}, \bibinfo{person}{Alan
  Karthikesalingam}, \bibinfo{person}{Mustafa Suleyman}, \bibinfo{person}{Greg
  Corrado}, {and} \bibinfo{person}{Dominic King}.}
  \bibinfo{year}{2019}\natexlab{}.
\newblock \showarticletitle{{Key Challenges for Delivering Clinical Impact with
  Artificial Intelligence}}.
\newblock \bibinfo{journal}{\emph{BMC medicine}} \bibinfo{volume}{17},
  \bibinfo{number}{1} (\bibinfo{year}{2019}), \bibinfo{pages}{195}.
\newblock


\bibitem[Li and Varakantham(2022)]%
        {li2022towards}
\bibfield{author}{\bibinfo{person}{Dexun Li} {and} \bibinfo{person}{Pradeep
  Varakantham}.} \bibinfo{year}{2022}\natexlab{}.
\newblock \showarticletitle{Towards Soft Fairness in Restless Multi-Armed
  Bandits}.
\newblock \bibinfo{journal}{\emph{arXiv preprint arXiv:2207.13343}}
  (\bibinfo{year}{2022}).
\newblock


\bibitem[Li et~al\mbox{.}(2019)]%
        {li2019combinatorial}
\bibfield{author}{\bibinfo{person}{Fengjiao Li}, \bibinfo{person}{Jia Liu},
  {and} \bibinfo{person}{Bo Ji}.} \bibinfo{year}{2019}\natexlab{}.
\newblock \showarticletitle{{Combinatorial Sleeping Bandits with Fairness
  Constraints}}.
\newblock \bibinfo{journal}{\emph{CoRR}}  \bibinfo{volume}{abs/1901.04891}
  (\bibinfo{year}{2019}).
\newblock
\showeprint[arxiv]{1901.04891}
\urldef\tempurl%
\url{http://arxiv.org/abs/1901.04891}
\showURL{%
\tempurl}


\bibitem[Liu and Zhao(2010)]%
        {liu2008indexability}
\bibfield{author}{\bibinfo{person}{Keqin Liu} {and} \bibinfo{person}{Qing
  Zhao}.} \bibinfo{year}{2010}\natexlab{}.
\newblock \showarticletitle{{Indexability of Restless Bandit Problems and
  Optimality of Whittle Index for Dynamic Multichannel Access}}.
\newblock \bibinfo{journal}{\emph{IEEE Transactions on Information Theory}}
  \bibinfo{volume}{56}, \bibinfo{number}{11} (\bibinfo{date}{Nov}
  \bibinfo{year}{2010}), \bibinfo{pages}{5547–5567}.
\newblock
\showISSN{1557-9654}
\urldef\tempurl%
\url{https://doi.org/10.1109/tit.2010.2068950}
\showDOI{\tempurl}


\bibitem[Liu et~al\mbox{.}(2003)]%
        {liu2003framework}
\bibfield{author}{\bibinfo{person}{Xin Liu}, \bibinfo{person}{Edwin~KP Chong},
  {and} \bibinfo{person}{Ness~B Shroff}.} \bibinfo{year}{2003}\natexlab{}.
\newblock \showarticletitle{A framework for opportunistic scheduling in
  wireless networks}.
\newblock \bibinfo{journal}{\emph{Computer networks}} \bibinfo{volume}{41},
  \bibinfo{number}{4} (\bibinfo{year}{2003}), \bibinfo{pages}{451--474}.
\newblock


\bibitem[Marseille and Kahn(2019)]%
        {marseille2019utilitarianism}
\bibfield{author}{\bibinfo{person}{Elliot Marseille} {and}
  \bibinfo{person}{James~G. Kahn}.} \bibinfo{year}{2019}\natexlab{}.
\newblock \showarticletitle{{Utilitarianism and the Ethical Foundations of
  Cost-Effectiveness Analysis in Resource Allocation for Global Health}}.
\newblock \bibinfo{journal}{\emph{Philosophy, Ethics, and Humanities in
  Medicine}} \bibinfo{volume}{14}, \bibinfo{number}{1} (\bibinfo{year}{2019}),
  \bibinfo{pages}{1--7}.
\newblock
\urldef\tempurl%
\url{https://doi.org/10.1186/s13010-019-0074-7}
\showDOI{\tempurl}


\bibitem[Mate et~al\mbox{.}(2020)]%
        {mate2020collapsing}
\bibfield{author}{\bibinfo{person}{Aditya Mate}, \bibinfo{person}{Jackson
  Killian}, \bibinfo{person}{Haifeng Xu}, \bibinfo{person}{Andrew Perrault},
  {and} \bibinfo{person}{Milind Tambe}.} \bibinfo{year}{2020}\natexlab{}.
\newblock \showarticletitle{{Collapsing Bandits and Their Application to Public
  Health Intervention}}.
\newblock \bibinfo{journal}{\emph{Advances in Neural Information Processing
  Systems (NeurIPS)}}  \bibinfo{volume}{33} (\bibinfo{year}{2020}).
\newblock


\bibitem[Mate et~al\mbox{.}(2021)]%
        {mate2021risk-aware}
\bibfield{author}{\bibinfo{person}{Aditya Mate}, \bibinfo{person}{Andrew
  Perrault}, {and} \bibinfo{person}{Milind Tambe}.}
  \bibinfo{year}{2021}\natexlab{}.
\newblock \showarticletitle{{Risk-Aware Interventions in Public Health:
  Planning with Restless Multi-Armed Bandits}}. In
  \bibinfo{booktitle}{\emph{20th International Conference on Autonomous Agents
  and Multiagent Systems (AAMAS)}}. \bibinfo{address}{London, UK}.
\newblock


\bibitem[Nesterov et~al\mbox{.}(2018)]%
        {nesterov2018lectures}
\bibfield{author}{\bibinfo{person}{Yurii Nesterov} {et~al\mbox{.}}}
  \bibinfo{year}{2018}\natexlab{}.
\newblock \bibinfo{booktitle}{\emph{Lectures on convex optimization}}.
  Vol.~\bibinfo{volume}{137}.
\newblock \bibinfo{publisher}{Springer}.
\newblock


\bibitem[Ni{\~n}o-Mora(2020)]%
        {nino2020verification}
\bibfield{author}{\bibinfo{person}{Jos{\'e} Ni{\~n}o-Mora}.}
  \bibinfo{year}{2020}\natexlab{}.
\newblock \showarticletitle{{A Verification Theorem for Threshold-Indexability
  of Real-State Discounted Restless Bandits}}.
\newblock \bibinfo{journal}{\emph{Mathematics of Operations Research}}
  \bibinfo{volume}{45}, \bibinfo{number}{2} (\bibinfo{year}{2020}),
  \bibinfo{pages}{465--496}.
\newblock


\bibitem[Papadimitriou and Tsitsiklis(1994)]%
        {papadimitriou1994complexity}
\bibfield{author}{\bibinfo{person}{Christos~H Papadimitriou} {and}
  \bibinfo{person}{John~N Tsitsiklis}.} \bibinfo{year}{1994}\natexlab{}.
\newblock \showarticletitle{{The Complexity of Optimal Queueing Network
  Control}}. In \bibinfo{booktitle}{\emph{Proceedings of IEEE 9th Annual
  Conference on Structure in Complexity Theory}}. IEEE,
  \bibinfo{pages}{318--322}.
\newblock


\bibitem[Prins et~al\mbox{.}(2020)]%
        {prins2020incorporating}
\bibfield{author}{\bibinfo{person}{Aviva Prins}, \bibinfo{person}{Aditya Mate},
  \bibinfo{person}{Jackson~A Killian}, \bibinfo{person}{Rediet Abebe}, {and}
  \bibinfo{person}{Milind Tambe}.} \bibinfo{year}{2020}\natexlab{}.
\newblock \showarticletitle{Incorporating Healthcare Motivated Constraints in
  Restless Bandit Based Resource Allocation}.
\newblock \bibinfo{journal}{\emph{NeurIPS 2020 Workshops: Challenges of Real
  World Reinforcement Learning, Machine Learning in Public Health (Best
  Lightning Paper), Machine Learning for Health (Best on Theme), Machine
  Learning for the Developing World}} (\bibinfo{year}{2020}).
\newblock
\urldef\tempurl%
\url{https://teamcore.seas.harvard.edu/publications/incorporating-healthcare-motivated-constraints-restless-bandit-based-resource}
\showURL{%
\tempurl}


\bibitem[Punjabi(2008)]%
        {punjabi2008epidemiology}
\bibfield{author}{\bibinfo{person}{Naresh~M Punjabi}.}
  \bibinfo{year}{2008}\natexlab{}.
\newblock \showarticletitle{The epidemiology of adult obstructive sleep apnea}.
\newblock \bibinfo{journal}{\emph{Proceedings of the American Thoracic
  Society}} \bibinfo{volume}{5}, \bibinfo{number}{2} (\bibinfo{year}{2008}),
  \bibinfo{pages}{136--143}.
\newblock


\bibitem[Qian et~al\mbox{.}(2016)]%
        {qian2016restless}
\bibfield{author}{\bibinfo{person}{Yundi Qian}, \bibinfo{person}{Chao Zhang},
  \bibinfo{person}{Bhaskar Krishnamachari}, {and} \bibinfo{person}{Milind
  Tambe}.} \bibinfo{year}{2016}\natexlab{}.
\newblock \showarticletitle{Restless poachers: Handling
  exploration-exploitation tradeoffs in security domains}. In
  \bibinfo{booktitle}{\emph{Proceedings of the 2016 International Conference on
  Autonomous Agents \& Multiagent Systems}}. \bibinfo{pages}{123--131}.
\newblock


\bibitem[Rajkomar et~al\mbox{.}(2018)]%
        {rajkomar2018ensuring}
\bibfield{author}{\bibinfo{person}{Alvin Rajkomar}, \bibinfo{person}{Michaela
  Hardt}, \bibinfo{person}{Michael~D Howell}, \bibinfo{person}{Greg Corrado},
  {and} \bibinfo{person}{Marshall~H Chin}.} \bibinfo{year}{2018}\natexlab{}.
\newblock \showarticletitle{{Ensuring Fairness in Machine Learning to Advance
  Health Equity}}.
\newblock \bibinfo{journal}{\emph{Annals of internal medicine}}
  \bibinfo{volume}{169}, \bibinfo{number}{12} (\bibinfo{year}{2018}),
  \bibinfo{pages}{866--872}.
\newblock


\bibitem[Rawls(1971)]%
        {rawls1971theory}
\bibfield{author}{\bibinfo{person}{John Rawls}.}
  \bibinfo{year}{1971}\natexlab{}.
\newblock \bibinfo{booktitle}{\emph{{A Theory of Justice}}
  (\bibinfo{edition}{1} ed.)}.
\newblock \bibinfo{publisher}{Belknap Press of Harvard University Press},
  \bibinfo{address}{Cambridge, Massachussets}.
\newblock
\showISBNx{0-674-88014-5}


\bibitem[Rhoades(1993)]%
        {rhoades1993herfindahl}
\bibfield{author}{\bibinfo{person}{Stephen~A Rhoades}.}
  \bibinfo{year}{1993}\natexlab{}.
\newblock \showarticletitle{{The Herfindahl-Hirschman Index}}.
\newblock \bibinfo{journal}{\emph{Fed. Res. Bull.}}  \bibinfo{volume}{79}
  (\bibinfo{year}{1993}), \bibinfo{pages}{188}.
\newblock


\bibitem[Rotenberg et~al\mbox{.}(2016)]%
        {rotenberg2016trends}
\bibfield{author}{\bibinfo{person}{Brian~W Rotenberg}, \bibinfo{person}{Dorian
  Murariu}, {and} \bibinfo{person}{Kenny~P Pang}.}
  \bibinfo{year}{2016}\natexlab{}.
\newblock \showarticletitle{{Trends in CPAP adherence over twenty years of data
  collection: a flattened curve}}.
\newblock \bibinfo{journal}{\emph{Journal of Otolaryngology-Head \& Neck
  Surgery}} \bibinfo{volume}{45}, \bibinfo{number}{1} (\bibinfo{year}{2016}),
  \bibinfo{pages}{1--9}.
\newblock


\bibitem[Rubner et~al\mbox{.}(2000)]%
        {rubner2000earth}
\bibfield{author}{\bibinfo{person}{Yossi Rubner}, \bibinfo{person}{Carlo
  Tomasi}, {and} \bibinfo{person}{Leonidas Guibas}.}
  \bibinfo{year}{2000}\natexlab{}.
\newblock \showarticletitle{The Earth Mover's Distance as a Metric for Image
  Retrieval}.
\newblock \bibinfo{journal}{\emph{International Journal of Computer Vision}}
  \bibinfo{volume}{40} (\bibinfo{date}{11} \bibinfo{year}{2000}),
  \bibinfo{pages}{99--121}.
\newblock
\urldef\tempurl%
\url{https://doi.org/10.1023/A:1026543900054}
\showDOI{\tempurl}


\bibitem[Sawyer et~al\mbox{.}(2011)]%
        {sawyer2011systematic}
\bibfield{author}{\bibinfo{person}{Amy~M Sawyer}, \bibinfo{person}{Nalaka~S
  Gooneratne}, \bibinfo{person}{Carole~L Marcus}, \bibinfo{person}{Dafna Ofer},
  \bibinfo{person}{Kathy~C Richards}, {and} \bibinfo{person}{Terri~E Weaver}.}
  \bibinfo{year}{2011}\natexlab{}.
\newblock \showarticletitle{{A systematic review of CPAP adherence across age
  groups: clinical and empiric insights for developing CPAP adherence
  interventions}}.
\newblock \bibinfo{journal}{\emph{Sleep medicine reviews}}
  \bibinfo{volume}{15}, \bibinfo{number}{6} (\bibinfo{year}{2011}),
  \bibinfo{pages}{343--356}.
\newblock


\bibitem[Scheunemann and White(2011)]%
        {scheunemann2011ethics}
\bibfield{author}{\bibinfo{person}{Leslie Scheunemann} {and}
  \bibinfo{person}{Douglas White}.} \bibinfo{year}{2011}\natexlab{}.
\newblock \showarticletitle{{The Ethics and Reality of Rationing in Medicine}}.
\newblock \bibinfo{journal}{\emph{Chest}}  \bibinfo{volume}{140}
  (\bibinfo{date}{12} \bibinfo{year}{2011}), \bibinfo{pages}{1625--32}.
\newblock
\urldef\tempurl%
\url{https://doi.org/10.1378/chest.11-0622}
\showDOI{\tempurl}


\bibitem[Srinivasan(2001)]%
        {srinivasan2001distributions}
\bibfield{author}{\bibinfo{person}{A. Srinivasan}.}
  \bibinfo{year}{2001}\natexlab{}.
\newblock \showarticletitle{{Distributions on Level-Sets with Applications to
  Approximation Algorithms}}. In \bibinfo{booktitle}{\emph{Proceedings of the
  42nd IEEE Symposium on Foundations of Computer Science}}
  \emph{(\bibinfo{series}{FOCS '01})}. \bibinfo{publisher}{IEEE Computer
  Society}, \bibinfo{address}{USA}, \bibinfo{pages}{588}.
\newblock
\showISBNx{0769513905}


\bibitem[Steimle and Denton(2017)]%
        {steimle2017markov}
\bibfield{author}{\bibinfo{person}{Lauren~N. Steimle} {and}
  \bibinfo{person}{Brian~T. Denton}.} \bibinfo{year}{2017}\natexlab{}.
\newblock \bibinfo{booktitle}{\emph{Markov Decision Processes for Screening and
  Treatment of Chronic Diseases}}.
\newblock \bibinfo{publisher}{Springer International Publishing},
  \bibinfo{address}{Cham}, \bibinfo{pages}{189--222}.
\newblock
\showISBNx{978-3-319-47766-4}
\urldef\tempurl%
\url{https://doi.org/10.1007/978-3-319-47766-4_6}
\showDOI{\tempurl}


\bibitem[Weber and Weiss(1990)]%
        {weber1990index}
\bibfield{author}{\bibinfo{person}{Richard~R Weber} {and}
  \bibinfo{person}{Gideon Weiss}.} \bibinfo{year}{1990}\natexlab{}.
\newblock \showarticletitle{{On an Index Policy for Restless Bandits}}.
\newblock \bibinfo{journal}{\emph{Journal of Applied Probability}}
  (\bibinfo{year}{1990}), \bibinfo{pages}{637--648}.
\newblock


\bibitem[Whittle(1988)]%
        {whittle1988restless}
\bibfield{author}{\bibinfo{person}{Peter Whittle}.}
  \bibinfo{year}{1988}\natexlab{}.
\newblock \showarticletitle{{Restless Bandits: Activity Allocation in a
  Changing World}}.
\newblock \bibinfo{journal}{\emph{Journal of applied probability}}
  \bibinfo{volume}{25}, \bibinfo{number}{A} (\bibinfo{year}{1988}),
  \bibinfo{pages}{287--298}.
\newblock


\end{thebibliography}

\newpage
\appendix
\onecolumn

\section{Notation}
\label{sec:notation}
In Table \ref{tab:notation}, we present an overview of the notation used in the paper. $[N]$ denotes the set $\{1, 2, \ldots, N\}$.
\begin{table}[!htb]
{\renewcommand{\arraystretch}{2}%
	\caption{Notation used in our model and notes on their interpretation.} \label{tab:notation}
	\resizebox{\textwidth}{!}{\begin{tabular}{|lll|}
		\hline 
		% MDP Variables
		\multicolumn{3}{|l|}{\textbf{MDP Variables} Here, timestep \(t\in [T] = \{1, 2, \dots T\}\) (subscript) and arm index \(i\in [N] = \{1, 2, \dots N\}\) (superscript) are implied.} \\ \hline
		\multicolumn{1}{|l|}{State space} & \multicolumn{1}{l|}{\(s\in\mathcal{S}=\{0,1\}\)} & \(s=\begin{cases} 1 & \textrm{arm is in the `good' state.}\\ 0  &  \text{else%, patient did not take their medication
		}\end{cases}\) \\ \hline
		\multicolumn{1}{|l|}{Belief space} & \multicolumn{1}{l|}{\(\begin{aligned}
			b \in\mathcal{B} &= [0,1]\\
			b_{t+1} &= \begin{cases}
			s_{t+1} & \textrm{if known} \\
			b_{t}P_{1,1}^{0} + (1-b_{t})P_{0,1}^{0} & \textrm{else}
			\end{cases}
			\end{aligned}\)} & {\renewcommand{\arraystretch}{1.25}%
			\begin{tabular}[c]{@{}l@{}}If an arm's true state is unknown, \\ the recursively computed belief state approximates it. \end{tabular}} \\ \hline
		\multicolumn{1}{|l|}{Action space} & \multicolumn{1}{l|}{\(a\in\mathcal{A} = \{0,1\}\)} & \(a=\begin{cases} 1 & \textrm{pull arm (i.e., provide intervention)}\\ 0  &  \text{else, don't pull}\end{cases}\) \\ \hline
		% MDP Functions
		\textbf{MDP Functions} &  &  \\ \hline
		\multicolumn{1}{|l|}{Transition function} & \multicolumn{1}{l|}{\(\begin{aligned}
			P \colon \mathcal{S}\times\mathcal{A}\times\mathcal{S} &\to [0,1]\\
			s_t, a_t, s_{t+1} &\mapsto \Pr(s_{t+1} \mid s_t, a_t)
			\end{aligned}\)} & {\renewcommand{\arraystretch}{1.25}%
			\begin{tabular}[c]{@{}l@{}}The probability of an arm \\ going from state \(s_t\) to \(s_{t+1}\), given action \(a_t\). \\ Equivalent (matrix) notation: \(P^{a_t}_{s_t,s_{t+1}}\). \end{tabular}} \\ \hline % P_a(s,s^\prime)
		\multicolumn{1}{|l|}{Reward function} & \multicolumn{1}{l|}{\(\begin{aligned}
			r \colon \mathcal{S}\text{ or }\mathcal{B} &\to \mathbb{R} % \\
			% r(s) &\mapsto s
			\end{aligned}\)} & {\renewcommand{\arraystretch}{1.25}%
			\begin{tabular}[c]{@{}l@{}}\(r(b)\) is used in computing the Whittle index. \end{tabular}} \\ \hline
		\multicolumn{1}{|l|}{Policy function} & \multicolumn{1}{l|}{\(\begin{aligned}
			\pi \colon \mathcal{S} &\to \mathcal{A}% \\
			%\pi(s) &\mapsto a
			\end{aligned}\)} & {\renewcommand{\arraystretch}{1.25}%
			\begin{tabular}[c]{@{}l@{}}A policy for actions.\\ The set of optimal policies is \(\pi^* \in \Pi^*\). \end{tabular}} \\ \hline
		% RMAB Variables
		\textbf{RMAB Variables} &  &  \\ \hline
		\multicolumn{1}{|l|}{Timestep} & \multicolumn{1}{l|}{\(\{t\in \mathbb{N} \mid t\leq T\}\)} & This timestep is implicit in the MDP. \\ \hline
		\multicolumn{1}{|l|}{Arm index} & \multicolumn{1}{l|}{\(i\in\{1,2,\dots,N\}\)} & {\renewcommand{\arraystretch}{1.25}%
			\begin{tabular}[c]{@{}l@{}}Each arm can represent a patient.\\ \(k\) arms can be pulled at any timestep \(t\).\end{tabular}} \\ \hline
		% RMAB objective functions
		\multicolumn{3}{|l|}{\textbf{Objective Functions} The objective is to find a policy \(\pi^* = \max_\pi \mathbb{E}_\pi [R(\cdot)]\).} \\ \hline
		\multicolumn{1}{|l|}{Discounted reward function} & \multicolumn{1}{l|}{\(\begin{aligned}
			R_\beta^\pi \colon \mathcal{S}^N &\to \mathbb{R}\\
			s_0^1,s_0^2,\dots,s_0^N &\mapsto  \sum_{i \in [N]} \sum_{t\in[T]} \beta^{t-1} r(s^i_t)
			\end{aligned}\)} & {\renewcommand{\arraystretch}{1.25}%
			\begin{tabular}[c]{@{}l@{}}\(\beta\in [0,1]\) is some \textit{discount parameter}.\end{tabular}} \\ \hline
		% RMAB constraint functions
	    \multicolumn{3}{|l|}{\textbf{Fairness-motivated Constraint Functions} 
	    %that motivate distributive fairness.
	    } \\ \hline
		\multicolumn{1}{|l|}{Integer periodicity } & \multicolumn{1}{l|}{\(\begin{aligned}
            \bigwedge_{j=0}^{\lceil\frac{T}{\nu} \rceil} \sum_{t=j\nu+1}^{(j+1)\nu} a_t^i \geq 1
            \end{aligned}\)} & {\renewcommand{\arraystretch}{1.25}%
			\begin{tabular}[c]{@{}l@{}} A form of time-indexed fairness. \\ Guarantees arm \(i\) is pulled at least once \\ within each period of \(\nu\) timesteps.\end{tabular}} \\ \hline
		\multicolumn{1}{|l|}{Minimum selection fraction } & \multicolumn{1}{l|}{\(\begin{aligned}
    		\bigwedge_{i\in[N]} \frac{1}{T}\sum_{t=1}^{T} a_t^i \geq \psi 
    		\end{aligned}\)} & {\renewcommand{\arraystretch}{1.25}%
			\begin{tabular}[c]{@{}l@{}} A form of time-indexed fairness. \\ Arm \(i\) should be pulled at least some \\ minimum fraction \(\psi \in (0,1)\) of timesteps.\end{tabular}} \\ \hline
		\multicolumn{1}{|l|}{Probabilistic} & \multicolumn{1}{l|}{\(\begin{aligned}
    		\bigwedge_{i\in[N]} \bigwedge_{t\in[T]} \Pr(a_t^i =1 \mid i,t) \in [\ell, u]
    		\end{aligned}\)} & {\renewcommand{\arraystretch}{1.25}%
			\begin{tabular}[c]{@{}l@{}}Pull each arm with probability \(p_i \in [\ell, u]\), \\ where \(\ell \in \left(0, \frac{k}{N}\right]\) and \(u \in \left[\frac{k}{N}, 1\right]\).\end{tabular}} \\ \hline
	\end{tabular}} 
}
\end{table}
\pagebreak 

%%%%%%%%%%%%%%%%%%%%%%%%%%%%%%%%%%%%%%%%%%%%%%%%%%%%%%%%%%
\section{Empirical Inequity in the Distribution of Actions under Whittle Index Policies}
\label{sec:AppWhittleUnFair}
Here, we present numerical results confirming \citet{prins2020incorporating}'s findings that \textsc{Threshold Whittle (TW)} tends to allocate pulls according to a bimodal distribution: a small subset of arms are pulled frequently, while others are largely ignored.

\textbf{Experimental Setup:} For each iteration, we generate \(N=2\) forward threshold-optimal arms and run \textsc{TW} for a \(T=365\) horizon simulation, where the budget constraint \(k=1\). We run \(1,000\) such iterations.

\textbf{Results:} In 515 out of 1,000 (51.5\%) simulations, the arms' Whittle indices never overlap, meaning that for \emph{any} combination of initial states, state transitions, and pulls, \textsc{TW} would pull one arm for all timesteps \(t\in T\) and completely ignore the second arm. We visualize one such case in Figure~\ref{fig:subsidy_vis}. Recall that \textsc{TW} precomputes the infimum subsidy \(m\) per arm and belief combination. Since belief is a function of last known state \(s\in \{0,1\}\) and time-since-seen \(u\in [T]\) (using the notation of~\citet{mate2020collapsing}), we plot the infimum subsidy of each arm-state combination with time-since-seen, \(u\), on the \(x\)-axis. There exists a horizontal line that divides the two arms, so arm \(i=1\) will be pulled for every timestep and arm \(i=2\) will never be pulled. 

% \begin{wrapfigure}{r}{0.5\linewidth}
\begin{figure}[!h]
  \centering
  \includegraphics[width=0.45\linewidth]{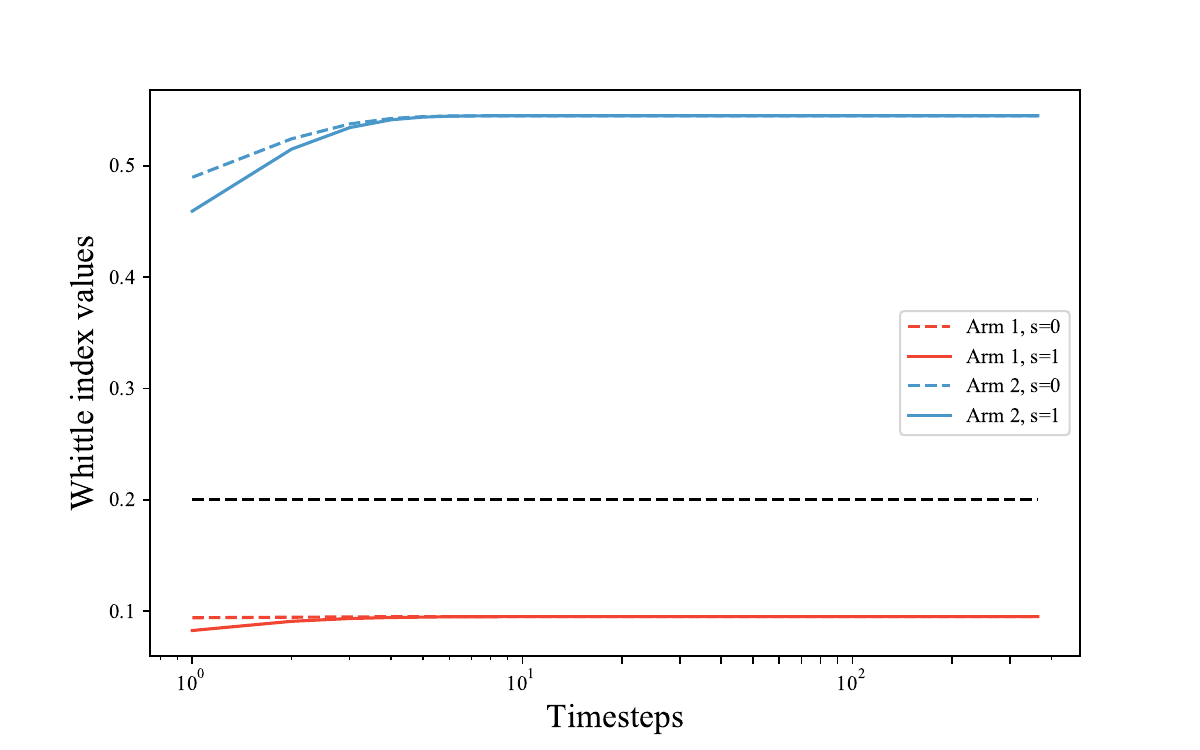}  
  \caption{The Whittle index values for Arm 1 and 2 can be separated by a horizontal line, meaning that (WLOG) Arm 1 will always be chosen over Arm 2 because its index value dominates.}
\label{fig:subsidy_vis}
\end{figure}
In order to modify the Whittle index to guarantee time-indexed fairness constraint satisfaction, one would need to ensure that no such horizontal line exists. Additionally, if we consider a specific form of time-indexed fairness known as an \emph{integer periodicity} constraint, which allows a decision-maker to guarantee that arm \(i\) is pulled at least once within each period of \(\nu\) days, the lines associated with the arms in Figure~\ref{fig:subsidy_vis} must cross before \(\nu\) timesteps elapse to guarantee fairness constraint satisfaction. 

Another perspective we can take is to ask: what's the smallest interval \(\nu_i\) for each arm \(i\) we could have specified such that \textsc{Threshold Whittle} would have satisfied the integer periodicity constraint? Note that this is retrospective, as there is no way to enforce this constraint at planning time.
\begin{figure}[!b]
  \centering
  \includegraphics[width=0.45\linewidth]{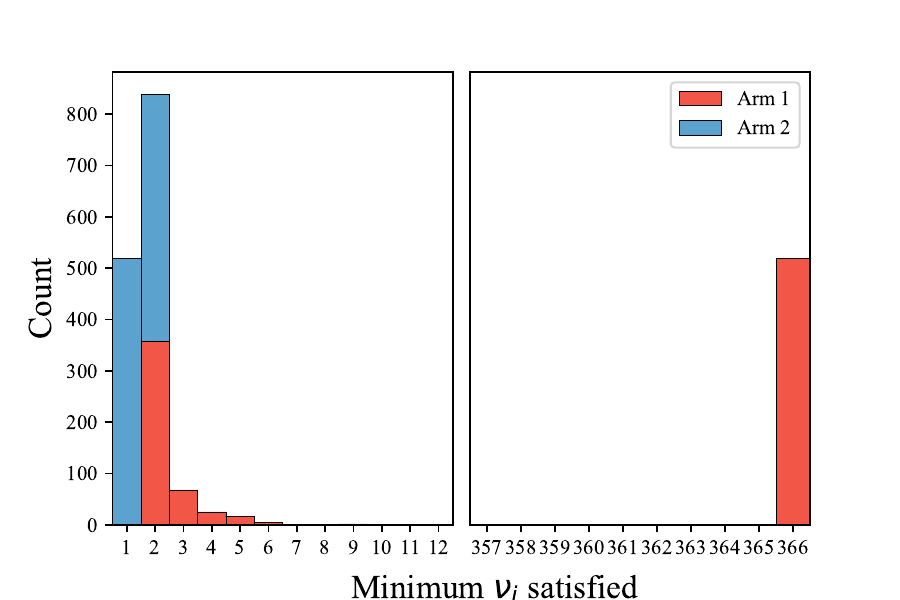}
  \caption{The smallest interval \(\nu_i\) such that \textsc{TW} satisfies an integer periodicity definition of time-indexed fairness, given \(N=2\) random arms. In over 50\% of iterations, no such fairness constraint satisfaction is possible (i.e., \(\exists i \text{ s.t. } \nu_i > T\)).}
\label{fig:nonindexability_min_nu_vis}
\end{figure}
We visualize the minimum such \(\nu_i\) in Figure~\ref{fig:nonindexability_min_nu_vis}. On the far right, we see the 515 cases where (WLOG) the second arm is never pulled---that is, the minimum \(\nu_i\) such that \textsc{Threshold Whittle} satisfies the hard integer periodicity constraint must be \emph{larger} than the horizon, \(T=365\). There is one case where arm \(i=2\) is pulled exactly once. In a majority of the remaining simulations, 
\textsc{TW} pulls each arm with approximately equal frequency.
\clearpage

%%%%%%%%%%%%%%%%%%%%%%%%%%%%%%%%%%%%%%%%%%%%%%%%%%%%%%%%%%
\section{Motivating a focus on distributive fairness}
In Section~\ref{sec:motivation} and Appendix~\ref{sec:AppWhittleUnFair}, we discuss and empirically demonstrate how Whittle index-based policies produce bimodal distributive outcomes in which a small subset of arms receive all available interventions, and the rest of the arms are ignored. Here, we note that there are many application domains where such skewed resource allocation may be perceived as unfair or undesirable by beneficiaries and decision-makers, thus motivating efforts to incentivize or guarantee distributive fairness. 

In the \emph{healthcare} examples we present in the main paper (i.e., where interventions support medication or CPAP adherence), resource constraints and variation in transition dynamics interact. A practical consequence is that a majority of patients will \emph{never} receive the beneficial intervention(s). This, in turn, means that their clinical outcomes will be strictly worse in expectation than they would be under a policy that guaranteed a non-zero probability of receiving the intervention at each timestep.

These considerations are not restricted to healthcare contexts. In \emph{poaching prevention}, the planner must be cognizant of the effect of choosing \textit{against} patrolling a particular area (i.e., pulling an arm) will have on poacher strategies \citep{qian2016restless}. In \emph{wireless scheduling}, multiple processes compete to transmit packets over a shared wireless channel. When scheduling the transmissions, the agent must not only maximize reward but also ensure the performance of critical applications codified in Quality of Service (QoS) guarantees \citep{liu2003framework, hou2009theory, li2019combinatorial}.

%%%%%%%%%%%%%%%%%%%%%%%%%%%%%%%%%%%%%%%%%%%%%%%%%%%%%%%%%%
\section{Intractability of Alternative Approaches}\label{sec:AppAltIntract}
In this section, we motivate the algorithmic design choices we have made when constructing \textsc{ProbFair} by discussing the feasibility of possible alternatives, including: (1) direct modification of the Whittle index to guarantee time-indexed fairness constraint satisfaction, and (2) a math programming-based approach. 

\subsection{Why not modify the Whittle index to guarantee time-indexed fairness constraint satisfaction?}
\label{sec:altModifyWhittle}
In Section~\ref{sec:fairnessIndexability}, we demonstrate that it is not possible to guarantee time-indexed fairness when arms are decoupled. If arms cannot be decoupled, the tractability of a Whittle index-based approach breaks down. Here, we discuss this topic in greater detail, and provide specific examples of possible Whittle index modifications. We also take this opportunity to emphasize that our focus is on \emph{guaranteeing} fairness rather than incentivizing it. \citet{mate2021risk-aware} provide an example of the latter, which we discuss and include as a comparison algorithm in Section~\ref{sec:pfVsFair}. 

To begin, recall that the efficiency of Whittle index-based policies stems from our ability to decouple arms when we are only concerned with maximizing total expected reward~\citep{whittle1988restless, weber1990index}. However, guaranteeing time-indexed fairness (as defined in Section~\ref{sec:rmabModel}) in the planning setting requires time-stamped record keeping. It is no longer sufficient to compute each arm's infimum subsidy in isolation and order the resulting set of values. Instead, for an optimal index policy to be efficiently computable, it must be possible to modify the value function (Equation~\ref{eqn:valueFunc}) so as to ensure that the infimum subsidy each arm would require in the absence of fairness constraints is minimally perturbed via augmentation or ``donation'', so as to maximize total expected reward while ensuring its own fairness constraint satisfaction or the constraint satisfaction of other arms, respectively, \emph{without} requiring input from other arms.

Plausible modifications include altering the conditions under which an arm receives the subsidy associated with passivity, $m$, or introducing a modified reward function, $r^{\prime}(b)$ that is capable of accounting for an arm's fairness constraint satisfaction status in addition to its state at time $t$. For example, we might use an indicator function to ``turn off'' the subsidy until arm $i$ has been pulled at least once within the interval in question, or increase reward as an arm's time-since-pulled value approaches the interval cut-off, so as to incentivize a constraint-satisfying pull. When these modifications are viewed from the perspective of a single arm, they \emph{appear} to have the desired effect: if no subsidy is received, it will be optimal to pull for all belief states; similarly, for a fixed $m$, as reward increases it will be optimal to pull for an increasingly large subset of the belief state space. 

Recall, however, that structural constraints ensure that when an arm is considered in isolation, the optimal action will \emph{always} be to pull. Whether or not arm $i$ is \emph{actually} pulled at time $t$ depends on how the infimum subsidy, $m$, it requires to accept passivity at time $t$ compares to the infimum subsidies required by other arms. Thus, any modification intended to \emph{guarantee} time-indexed fairness constraint satisfaction must be able to alter the ordering \emph{among} arms, such that any arm $i$ which would otherwise have a subsidy with rank $>k$ when sorted in descending order will now be in the top-$k$ arms. Even if we were able to construct such a modification for a single arm without requiring system state, if \emph{every} arm had this same capability, then a new challenge would arise: we would be unable to distinguish among arms, and arbitrary tie-breaking could again jeopardize fairness constraint satisfaction. 

If it is not possible to decouple arms, then we must consider them in tandem. \citet{papadimitriou1994complexity} prove that the RMAB problem is PSPACE-complete even when transition rules are action-dependent but deterministic, via reduction from the halting problem. 

\subsection{Why not use a math-programming approach?}
\label{sec:mathProg}
Our constrained maximization problem can be readily formulated as an integer program (IP) with a totally unimodular (TU) constraint matrix. However, this approach is intractable because the objective function coefficients of this IP cannot be efficiently enumerated. To support this intractability claim, we begin by presenting an integer program (IP) that maximizes total expected reward under both budget \emph{and} time-indexed fairness constraints, for problem instances with feasible hyperparameters. We then prove that any problem instance with feasible hyperparameters yields a totally unimodular (TU) constraint matrix, which ensures that the linear program (LP) relaxation of our IP will yield an integral solution. We proceed to demonstrate that tractability issues arise because we incur an exponential dependency on the time horizon, $T$, when we construct the IP's objective function coefficients. We conclude by comparing \textsc{ProbFair} to the IP for small values of $N$ and $T$. 

\subsubsection{Integer Program Formulation}
\label{sec:ip} 

To leverage a math programming approach for our constrained reward maximization task, we seek to construct an integer program (IP) whose solution is the policy \(\vec{x}\in \{0,1\}^{N\lvert\mathcal{A}\rvert T}\). We require this policy to be reward-maximizing, subject to the guaranteed satisfaction of both budget and time-indexed fairness constraints. To begin, let each decision variable $x_{i,a,t} \in \{0,1\}$ represent whether or not we take action $a \in \mathcal{A}=\{0,1\}$ for arm $i\in[N]$ at time $t\in[T]$. Then, let each objective function coefficient $c_{i,a,t}$ represent the expected reward associated with an arm-action-timestep combination.

To formalize the objective function, recall that the agent seeks to maximize total expected reward, \(\mathbb{E}_\pi[R(\cdot)]\). For clarity of exposition, we specifically consider the linear global reward function \(R(r(s)) = \sum_{i=1}^N \sum_{t=0}^T s_t^i \). Note that this implies the discount rate, $\beta = 1$; however, the approach outlined here can be extended in a straightforward manner for $\beta \in (0,1)$. In order to compute the expected reward associated with taking action $a$ for arm $i$ at time $t$, we must consider: (1) what state is the arm currently in (i.e., what is the realized value of $s_t^i \in \{0,1\}$)? (2) when the arm transitions from $s_t$ to $s_{t+1}$ by virtue of taking action $a$, what reward, $r(\cdot)$, should we expect to earn?

Because we define $r(s)=s$, (2) can be reframed as: what is the probability $p(s_{t+1} = 1 | s_t^i, a_t^i)$ that action $a$ causes a transition from $s_t$ to the adherent state? Because each arm's state at time $t$ is stochastic, depending not only on the sequence of actions taken in previous timesteps, but the associated set of stochastic transitions informed by the arm's underlying MDP, each coefficient of our objective function must be computed as the expectation over the possible values of $s_t \in \mathcal{S}$:
\begin{align}
    \vec{c} &= \mathbb{E}_s[p(s_{t+1} = 1 | x_{i,a,t},  s_t)] &\forall i,a,t \in [N] \times \mathcal{A} \times [T]\\
    & = \frac{1}{2^t}\sum_{s \in \mathcal{S}} p(s_t = s)\sum_{s^\prime \in S}p(s_{t+1} = s^\prime | x_{i,a,t}, s_t = s)r(s^\prime) &\forall i,a,t \in [N] \times \mathcal{A} \times [T]\\
    & = \frac{1}{2^t}\sum_{s \in \mathcal{S}} p(s_t = s)p(r(s_{t+1}) = 1 | x_{i,a,t}, s_t = s) & \forall i,a,t \in [N] \times \mathcal{A} \times [T]\\
    & = \frac{1}{2^t}\sum_{s \in \mathcal{S}} p(s_t = s)p(s_{t+1} = 1 | x_{i,a,t}, s_t = s) &\forall i,a,t \in [N] \times \mathcal{A} \times [T] \label{eqn:IPcoeffs}
\end{align}

Within the context of this IP, the time-indexed fairness constraint  we introduce in Section~\ref{sec:rmabModel} can be more specifically defined as either an \emph{integer periodicity} or \emph{minimum selection fraction} constraint. We formalize each of these below:

The \textbf{integer periodicity constraint} allows a decision-maker to guarantee that arm \(i\) is pulled at least once within each period of \(\nu\) days. We define this constraint as a function $g$, over the vector of actions, $\vec{a}^i$ associated with arm $i$, and user-defined interval length $\nu \in [1, T]$: 
 \begin{align}
    g(\vec{a}^i) = &\sum_{t=j\nu+1}^{(j+1)\nu} a_t^i \geq 1 \\ 
    & \forall j \in \left\{0, 1, 2, \dots \left\lceil \frac{T}{\nu}\right\rceil \right\} \text{; }\forall i \in \{1, 2, \dots N\} \nonumber
    \label{eqn:minSelFrac}
\end{align}   

The \textbf{minimum selection fraction constraint} introduced by~\citet{li2019combinatorial} forces the agent to pull arm \(i\) at least a minimum fraction, $\psi \in (0,1)$, of the total number of steps, but is agnostic to how these pulls are distributed over time. We define this constraint, $g^\prime$, as a function over the vector of actions, $\vec{a}^i$ associated with arm $i$ and user-defined $\psi$: 
\begin{align}
    g^\prime(\vec{a}^i) = \frac{1}{T}\sum_{t=1}^{T} a_t^i \geq \psi \ \forall i \in \{1,2, \dots N\}
\end{align}

The resulting integer program is given by: 
\begin{align}
\text{max} \quad &
c^T x \label{alg:IP}\\
\text{s.t.} \quad & \sum_{a=1}^{|A|} x_{i,a,t} = 1 \quad &\forall i \in [N], t \in [T] \quad & \begin{tabular}[t]{@{}l@{}}\textcolor{RoyalBlue}{(a) Select exactly one action per}\\ \textcolor{RoyalBlue}{\;\;\;\;  arm $i$ at each $t$}\end{tabular} \nonumber \\
\text{\textcolor{cadetBlue}{\textit{if int. per: }}}&\sum_{t \in I_j} x_{i,1,t} \geq 1 &\quad \forall j \in \left\{0,1,\dots \frac{T-\nu}{\nu}\right\} \quad & \begin{tabular}[t]{@{}l@{}}\textcolor{RoyalBlue}{(b.i) Pull each arm $i$ at least once}\\ \textcolor{RoyalBlue}{\;\;\;\;\;\; during each interval of length $\nu$}\end{tabular} \nonumber \\
\text{\textcolor{cadetBlue}{\textit{if min. sel: }}} &\frac{1}{T}\sum_{t=1}^{T} x_{i,1,t} \geq \psi \quad  &\psi \in (0,1), \forall i \in N \quad & \begin{tabular}[t]{@{}l@{}}\textcolor{RoyalBlue}{(b.ii) Pull each arm $i$ at least a}\\ \textcolor{RoyalBlue}{\;\;\;\;\;\;\; minimum fraction $\psi$ of $T$ rounds}\end{tabular} \nonumber \\
&\sum_{i=1}^{N} x_{i,1,t} = k \quad &\forall t \in [T] \quad & \text{\textcolor{RoyalBlue}{(c) Pull exactly $k$ arms at each $t$}} \nonumber \\
&x_{i,a,t} \in \{0,1\} \quad &\forall i\in [N], a \in \mathcal{A}, t \in [T] \quad & \begin{tabular}[t]{@{}l@{}}\textcolor{RoyalBlue}{(d) Each arm-action-timestep choice}\\ \textcolor{RoyalBlue}{\;\;\;\; is a binary decision variable}\end{tabular}\nonumber
\end{align} 

\subsubsection{LP Relaxation and Integrality of Solution}
\label{sec:iptu}
We now prove that the IP we have formulated in Section~\ref{sec:ip} has an attractive property: namely, any feasible problem instance will produce a totally unimodular constraint matrix. Our proof leverages a theorem introduced by \citeauthor{ghouila1962caracterisation} and restated below for convenience, which can be used to determine whether a matrix, $A \in \mathbb{R}^{m \times n}$ is totally unimodular: 

\begin{lemma}(\citet{ghouila1962caracterisation})
A matrix $A \in \mathbb{Z}^{m \times n}$ is totally unimodular (TU) if and only if for every subset of the rows $R \subseteq [m]$, there is a partition $R = R_1 \cup R_2$ such that for every $j \in [n]$, 
\begin{equation}
    \sum_{i \in R_1} A_{ij} - \sum_{i \in R_2} A_{ij} \in \{-1,0,1\}
\end{equation}
\label{lemma:gh}
\end{lemma}
\begin{theorem}
Within the context of the integer program outlined in Appendix \ref{sec:ip}, any feasible problem instance will produce a constraint matrix that is totally unimodular (TU).
\end{theorem}

\begin{proof}
To begin, we establish the dimensions of any such constraint matrix $\mathbf{A}$ and note the maximum possible column-wise sum that each of its component submatrices may contribute. Note that the minimum selection fraction constraint (b.ii in Equation~\ref{alg:IP}), which requires the agent to pull each arm $i$ at least a minimum fraction, $\psi \in (0,1)$, of $T$ rounds, can be thought of as a special case of the integer periodicity constraint, (b.i in Equation~\ref{alg:IP}),
where $\nu = T$ and each arm must be pulled at least \(\lceil T\psi\rceil\) times. As such, we assume that at most one of the time-indexed fairness constraints can be selected, and focus on the more general of the two, which is the integer periodicity constraint. For notational convenience, we refer to constraints by their alphabetic identifiers. Let (b) represent the integer periodicity constraint, and define a function $\upvarphi: r \in \mathbf{R} \subseteq \mathbf{A} \mapsto e \in \{a, b, c\}$ that maps each row to its corresponding constraint type. 

First, recall that each $x_{i,a,t}$ represents a single binary decision variable, and corresponds to a column in $\mathbf{A}$. There are  $N \times |\mathcal{A}| \times T$ such columns. Next, note that constraint (a) enforces the requirement that we select \emph{exactly} one action per arm per timestep. Formally, $\forall i, t \in N \times T$, $\exists! a \in \mathcal{A}$ s.t. $x_{i,a,t} = 1$. Correspondingly, $\forall a^{\prime} \in \mathcal{A} \setminus a, x_{i, a^{\prime}, t} = 0$. The column vectors of the associated sub-matrix, $\mathbf{A}_a \in \mathbb{Z}^{NT \times N|\mathcal{A}|T}$, are indexed by disjoint $(i, a, t) \in N \times |\mathcal{A}| \times T$; thus, each column vector contains a single non-zero entry and for $\mathbf{R}_a \subseteq \mathbf{A}_a$, taking the column-wise sum will yield a vector $\vec{v} \in \mathbb{Z}^{N|\mathcal{A}|T}$ with every entry equal to $1$.

In a similar vein, equity constraint (b) enforces the requirement that we must pull each arm, $i$ at least once during each interval $I_j$ of length $\nu_i$. Within the associated sub-matrix, $\mathbb{A}_{b} \in \mathbb{Z}^{ N\lceil\frac{T}{\nu_i}\rceil \times N|\mathcal{A}|T}$, each column that corresponds to a passive action (e.g., $x_{i, a=0, t}$) will have \emph{only} zero-valued entries, since passive action decision variables are not impacted by constraint (b). Conversely, each column that corresponds to an active action (e.g., $x_{i, a=1, t}$) will have a single non-zero entry. Each active action column corresponding to a specific arm-timestep can be mapped to exactly one interval. Thus, for $\mathbf{R}_b \subseteq \mathbf{A}_b$, taking the column-wise sum will yield a vector $\vec{v} \in \mathbb{Z}^{N|\mathcal{A}|T}$ with every entry taking a value $\in \{0, 1\}$. 

The budget constraint (c) enforces the requirement that we must pull exactly $k$ of the $N$ arms at each timestep. Much like equity constraint (b), only columns corresponding to active actions are impacted. Thus, within the associated sub-matrix, $\mathbf{A}_{c} \in \mathbb{Z}^{ T \times N|\mathcal{A}|T}$, each column that corresponds to a passive action (e.g., $x_{i, a=0, t}$) will have \emph{only} zero-valued entries, while each column that corresponds to an active action (e.g., $x_{i, a=1, t}$) can be mapped to a single timestep, and will have a single non-zero entry. Thus, for $\mathbf{R}_c \subseteq \mathbf{A}_c$, taking the column-wise sum also yield a vector $\vec{v} \in \mathbb{Z}^{N|\mathcal{A}|T}$ with every entry taking a value $\in \{0, 1\}$. 

The complete constraint matrix $\mathbf{A}$ thus contains $NT + N\lceil\frac{T}{\nu_i}\rceil + T$ rows. Three possible cases arise when we consider every subset of these rows: (1) $\mathbf{R} \subsetneq \mathbf{A} = \emptyset$; (2) $\mathbf{R} \subsetneq \mathbf{A}; \mathbf{R} \cap \mathbf{A} \neq \emptyset$; (3) $\mathbf{R} \subseteq \mathbf{A}$. 

\begin{case} 
$\mathbf{R} \subsetneq \mathbf{A} = \emptyset$. To satisfy Lemma \ref{lemma:gh}, partition $\mathbf{R}$ such that $\mathbf{R} = \mathbf{R}_1 \cup \mathbf{R}_2 = \emptyset \cup \emptyset$. Then, for every $j \in [n]$, $$\sum_{i \in \mathbf{R}_1} \mathbf{A}_{ij} - \sum_{i \in \mathbf{R}_2} \mathbf{A}_{ij} = 0 - 0; 0 \in \{-1,0,1\} \text{ \qedsymbol}$$
\end{case}

\begin{case}
\label{case:R2}
$\mathbf{R} \subsetneq \mathbf{A}; \mathbf{R} \cap \mathbf{A} \neq \emptyset$. If we consider $\cup_{r \in \mathbf{R}} \upvarphi(r)$, there are $\sum_{k=1}^{3} {3 \choose k}$ possible sets of observed constraint types: $\{a\} \lor \{b\} \lor \{c\} \lor \{a,b\} \lor \{a,c\} \lor \{b,c\} \lor \{a,b,c\}$. 

\begin{enumerate}
    \label{case:R21}
    \item If $|\cup_{r \in \mathbf{R}} \upvarphi(r)| = 1$, then any partition of $\mathbf{R}$ will satisfy Lemma \ref{lemma:gh}. Without loss of generality, let each row $r \in \mathbf{R}$ belong to $\mathbf{R}_1$ and $\mathbf{R}_2 = \emptyset$. 
    \begin{enumerate}[label=\roman*.]
        % {a}
        \item If $\cup_{r \in \mathbf{R}} \upvarphi(r) = \{a\}$, taking the column-wise sum of $\mathbf{R}_1$ will yield a vector $\vec{v} \in \mathbb{Z}^{N|\mathcal{A}|T}$ with every entry $\in \{1\} \text{ if } \mathbf{R} \subseteq \mathbf{A}_a, \text{ and } \in \{0, 1\} \text{ otherwise}$. Thus, $\forall j \in [N|\mathcal{A}|T]$, $$\sum_{i \in \mathbf{R}_1} \mathbf{A}_{ij} - \sum_{i \in \mathbf{R}_2} \mathbf{A}_{ij} = 0 - 0 \lor 1-0; \{0,1\} \subsetneq \{-1,0,1\} \text{ \qedsymbol}$$
        % {b} or {c}
        \item If $\cup_{r \in \mathbf{R}} \upvarphi(r) = \{b\} \lor \{c\}$, taking the column-wise sum of $\mathbf{R}_1$ will yield a vector $\vec{v} \in \mathbb{Z}^{N|\mathcal{A}|T}$ with every entry corresponding to a passive action column $\in \{0\}$ and every entry corresponding to an active action column $\in \{1\} \text{ if } \mathbf{R} \subseteq \mathbf{A}_{b \lor c}, \text{ and } \in \{0, 1\} \text{ otherwise}$. Thus, $\forall j \in [N|\mathcal{A}|T]$, $$\sum_{i \in \mathbf{R}_1} \mathbf{A}_{ij} - \sum_{i \in \mathbf{R}_2} \mathbf{A}_{ij} = 0 - 0 \lor 1-0; \{0,1\} \subsetneq \{-1,0,1\} \text{ \qedsymbol}$$
    \end{enumerate}
    \item If $|\cup_{r \in \mathbf{R}} \upvarphi(r)| = 2$, without loss of generality, partition as follows: sort the elements $\in \cup_{r \in \mathbf{R}} \upvarphi(r)$ lexicographically, and let $\mathbf{R}_1 = \{r | \upvarphi(r) = \min \cup_{r \in \mathbf{R}} \upvarphi(r)\}$ and $\mathbf{R}_2 = \mathbf{R} \setminus \mathbf{R}_1$. Per Case \hyperref[case:R2]{2.1 (i) and (ii)}, taking the column-wise sums of $\mathbf{R}_1$ and $\mathbf{R}_2$ will yield two vectors, $\vec{v_1}, \vec{v_2} \in \mathbb{Z}^{N|\mathcal{A}|T}$, each of which will contain only entries $\in \{0, 1\}$. Thus, $\forall j \in [N|\mathcal{A}|T]$, $$\sum_{i \in \mathbf{R}_1} \mathbf{A}_{ij} - \sum_{i \in \mathbf{R}_2} \mathbf{A}_{ij} = 0 - 0 \lor 0-1 \lor 1-0 \lor 1-1; \{-1,0,1\} \subseteq \{-1,0,1\} \text{ \qedsymbol}$$
    
    \item If $|\cup_{r \in \mathbf{R}} \upvarphi(r)| = 3$, partition according to Algorithm \ref{alg:partition}:

    \begin{algorithm}[H]
    \caption{Partition for Case 2.3}
    \label{alg:partition}
    \begin{algorithmic}[1] 
    \Procedure{Partition}{$N$, $\mathcal{A}$, $T$, $\mathbf{R}$}
    \State{$C_{R_1}; C_{R_2} \gets \{0\}^{N|\mathcal{A}|T}$} \Comment{\textcolor{RoyalBlue}{Initialize two sets of counters}}
    \State{$\mathbf{R}_1; \mathbf{R}_2 \gets \emptyset$}
    
    \For{$\text{element } e \in \{a,b,c\}$}
        \State{$\mathbf{R}_{e} \gets \{r|\upvarphi(r) = e\}$} \Comment{\textcolor{RoyalBlue}{Use $\upvarphi$ to partition the rows of $\mathbf{R}$ by constraint type}}
    \EndFor

    \For{$r \in \mathbf{R}_a$}
        \State{$\mathbf{R}_1 \gets \mathbf{R}_1 \cup r$}\Comment{\textcolor{RoyalBlue}{Let $r \in \mathbf{R}_1$}}
        \For{$i \in 0:N|\mathcal{A}|T$}
            \If{$r[i] > 0$} \Comment{\textcolor{RoyalBlue}{For each non-zero entry $\in r$}}
                \State{$C_{R_1}[i] \gets C_{R_1}[i] + 1$}~\Comment{\textcolor{RoyalBlue}{Increment corresponding $C_{R_1}$ counter}}
            \EndIf
        \EndFor
    \EndFor 
    
    \For{$r \in \mathbf{R}_b$}
        \State{\texttt{flag} $\gets $ \texttt{false}}
            \For{$i \in 0:N|\mathcal{A}|T$}
                \If{$r[i] > 0 \land C_{R_1}[i] > 0$}\Comment{\textcolor{RoyalBlue}{If any non-zero element $r_i$ has $C_{R_1} >0$}}
                    \State{\texttt{flag} $\gets $ \texttt{true}}\Comment{\textcolor{RoyalBlue}{Set flag to \texttt{true}}}
                \EndIf
            \EndFor 
        
        \If{\texttt{flag}}
            \State{$\mathbf{R}_2 \gets \mathbf{R}_2 \cup r$}\Comment{\textcolor{RoyalBlue}{Let $r \in \mathbf{R}_2$}}
            \For{$i \in 0:N|\mathcal{A}|T$}
                \If{$r[i] > 0$}\Comment{\textcolor{RoyalBlue}{For each non-zero entry $\in r$}}
                \State{$C_{R_2}[i] \gets C_{R_2}[i] + 1$}~\Comment{\textcolor{RoyalBlue}{Increment corresponding $C_{R_2}$ counter}}
                \EndIf
            \EndFor 
        \Else 
            \State{$\mathbf{R}_1 \gets \mathbf{R}_1 \cup r$}\Comment{\textcolor{RoyalBlue}{Let $r \in \mathbf{R}_1$}}
            \For{$i \in 0:N|\mathcal{A}|T$}
                \If{$r[i] > 0$}\Comment{\textcolor{RoyalBlue}{For each non-zero entry $\in r$}}
                \State{$C_{R_1}[i] \gets C_{R_1}[i] + 1$}~\Comment{\textcolor{RoyalBlue}{Increment corresponding $C_{R_1}$ counter}}
                \EndIf
            \EndFor 
        \EndIf 
    \EndFor

    \For{$r \in \mathbf{R}_c$}
    \State{$\mathbf{R}_2 \gets \mathbf{R}_2 \cup r$}\Comment{\textcolor{RoyalBlue}{Let $r \in \mathbf{R}_2$}}
        \For{$i \in 0:N|\mathcal{A}|T$}
            \If{$r[i] > 0$}\Comment{\textcolor{RoyalBlue}{For each non-zero entry $\in r$}}
                \State{$C_{R_2}[i] \gets C_{R_2}[i] + 1$}\Comment{\textcolor{RoyalBlue}{Increment corresponding $C_{R_2}$ counter}}
            \EndIf
        \EndFor
    \EndFor

    \Return{$\mathbf{R}_1; \mathbf{R}_2$}
    \EndProcedure
    \end{algorithmic}
    \end{algorithm}
    \setlength{\textfloatsep}{0.1cm}
    \setlength{\floatsep}{0.1cm}
\end{enumerate}

Taking the column-wise sums of the resulting $\mathbf{R}_1$ and $\mathbf{R}_2$ will yield two vectors $\vec{v_1}, \vec{v_2} \in \mathbb{Z}^{N|\mathcal{A}|T}$, which can contain entries $\in \{0,1\}$ and $\{0,1,2\}$, respectively. Note that since $\vec{v_2}$ is constructed by taking only rows with constraint types $\in \{b, c\}$, only entries corresponding to active action columns can take values $>1$. Moreover,  $\forall j \in [N|\mathcal{A}|T], \sum_{i \in \mathbf{R}_2} \mathbf{A}_{ij} = 2 \rightarrow \sum_{i \in \mathbf{R}_1} \mathbf{A}_{ij} = 1$. Thus, $\forall j \in [N\lvert\mathcal{A}\rvert T]$, $$\sum_{i \in \mathbf{R}_1} \mathbf{A}_{ij} - \sum_{i \in \mathbf{R}_2} \mathbf{A}_{ij} = 0 - 0 \lor 0-1 \lor 1-0 \lor 1-1 \lor 1-2; \{-1,0,1\} \subseteq \{-1,0,1\} \text{ \qedsymbol}$$
\end{case}

\begin{case} $R \subseteq A$. Since $|\cup_{r \in \mathbf{R}} \upvarphi(r)| = 3$, proceed as outlined in Case 2.3. Only a slight modification is required: since $\mathbf{R}$ is now equal to $\mathbf{A}$, taking the column-wise sums of the resulting $\mathbf{R}_1$ and $\mathbf{R}_2$ will yield two vectors $\vec{v_1}, \vec{v_2} \in \mathbb{Z}^{N\lvert\mathcal{A}\rvert T}$, which can contain entries $\in \{1\}$ and $\{0,2\}$, respectively. Thus, $\forall j \in [N\lvert\mathcal{A}\rvert T]$, $$\sum_{i \in \mathbf{R}_1} \mathbf{A}_{ij} - \sum_{i \in \mathbf{R}_2} \mathbf{A}_{ij} = 1-0 \lor 1-2; \{-1,1\} \subsetneq \{-1,0,1\}$$ %\text{ \qedsymbol}
\end{case}
\vspace{-2em}
\end{proof}
\vspace{-1em}

\subsubsection{Enumeration of Objective Function Coefficients}
\label{sec:ipObjFunc}
The key challenge we encounter when we seek to enumerate the IP outlined in Section~\ref{sec:ip} is that exact computation of the objective function coefficients, \(\vec{c}\in \mathbb{R}^{N\lvert \mathcal{A}\rvert T}\) is intractable. Each arm contributes $|\mathcal{A}|\times T$ coefficients, and while calculation is trivially parallelizable over arms, we must consider a probability tree like the one in Figure~\ref{tikz:IPcoeffs} for each arm. % \todo{edit to match new positioning of this section}
\tikzstyle{level 1}=[level distance=5cm, sibling distance=5.8cm]
\tikzstyle{level 2}=[level distance=5cm, sibling distance=1.5cm]
\tikzstyle{level 3}=[level distance=2cm, sibling distance=.25cm]
\tikzstyle{bag} = [text width=4em, text centered]
\tikzstyle{iatnode} = [text centered, circle, thick, draw=RoyalBlue, radius=1pt, fill=white]
\tikzstyle{end} = [circle, minimum width=3pt,fill, inner sep=0pt]
\tikzstyle{endnotnecessary} = [circle, minimum width=3pt, inner sep=0pt]
\tikzstyle{start} = [circle, minimum width=3pt, inner sep=0pt]
\tikzstyle{lookahead} = [RoyalBlue,thick,draw]
\tikzstyle{notnecessary} = [dashed,draw]
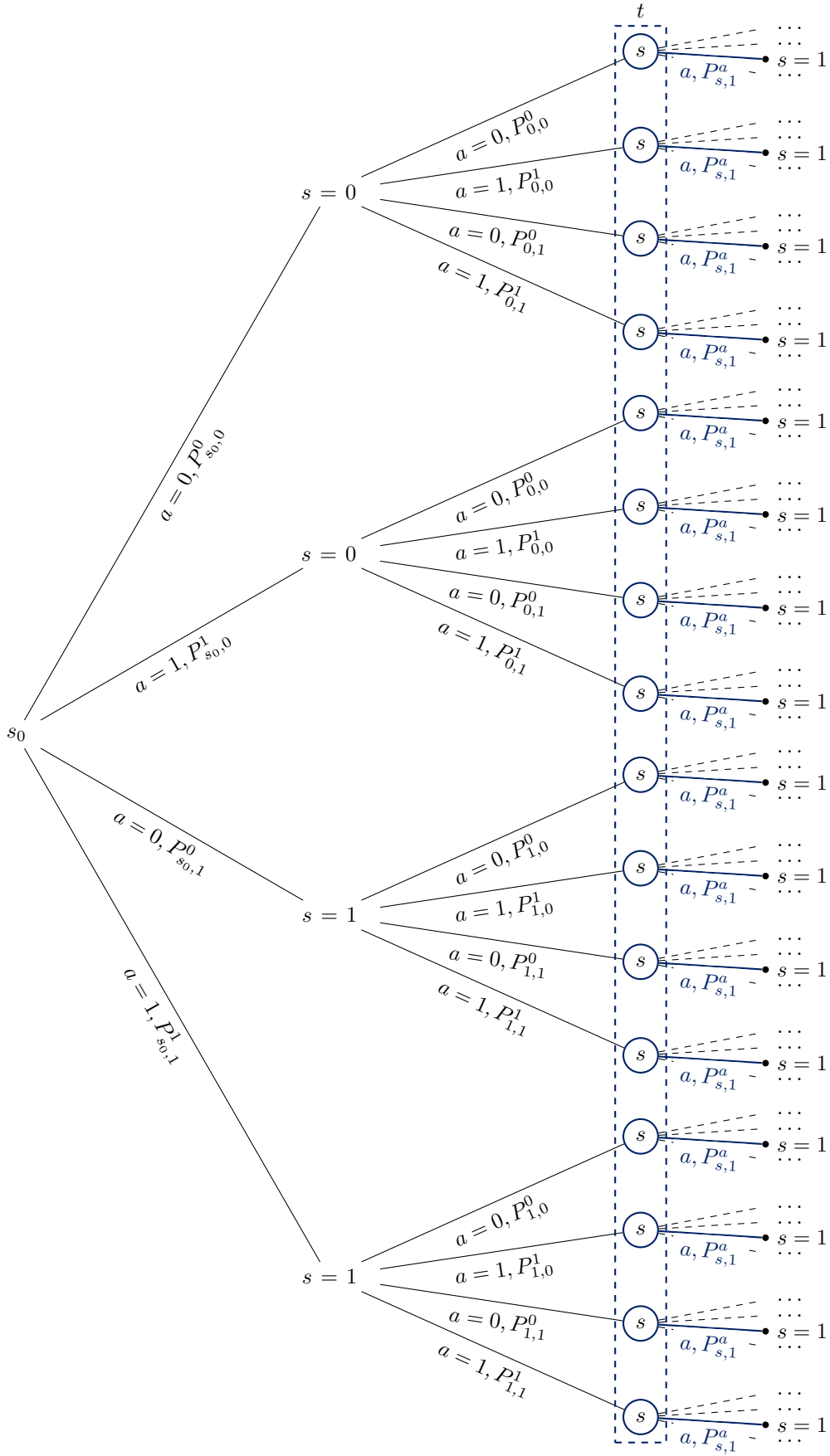
\begin{figure}[p]
\caption{Illustration of the probability tree for finding the coefficient corresponding to \(x_{i,a=0,t=2}\).}
\label{tikz:IPcoeffs}
\centering
\begin{tikzpicture}[grow=right, sloped]
\node(root)[bag] {$s_0$}
    child {
        node[bag] {$s=1$}        
            child {
                node[iatnode] {$s$} 
                    child {
                        node[endnotnecessary, label=right:
                            {$\dots$}] {}
                        edge from parent[notnecessary]
                        node[above] { }
                        node[below] { }
                    }
                    child {
                        node[end, label=right:
                            {$s=1$}] {}
                        edge from parent[lookahead]
                        node[above,fill=white] { }
                        node[below,fill=white] {$a, P^a_{s,1}$}
                    }
                    child {
                        node[endnotnecessary, label=right:
                            {$\dots$}] {}
                        edge from parent[notnecessary]
                        node[above] { }
                        node[below] { }
                    }
                    child {
                        node[endnotnecessary, label=right:
                            {$\dots$}] {}
                        edge from parent[notnecessary]
                        node[above] { }
                        node[below] { }
                    }
                edge from parent
                node[above] { }
                node[below]  {$a=1, P^1_{1,1}$}
            }
            child {
                node[iatnode] {$s$} 
                    child {
                        node[endnotnecessary, label=right:
                            {$\dots$}] {}
                        edge from parent[notnecessary]
                        node[above] { }
                        node[below] { }
                    }
                    child {
                        node[end, label=right:
                            {$s=1$}] {}
                        edge from parent[lookahead]
                        node[above, fill=white] { }
                        node[below, fill=white] {$a, P^a_{s,1}$}
                    }
                    child {
                        node[endnotnecessary, label=right:
                            {$\dots$}] {}
                        edge from parent[notnecessary]
                        node[above] { }
                        node[below] { }
                    }
                    child {
                        node[endnotnecessary, label=right:
                            {$\dots$}] {}
                        edge from parent[notnecessary]
                        node[above] { }
                        node[below] { }
                    }
                edge from parent
                node[above] { }
                node[below]  {$a=0, P^0_{1,1}$}
            }
            child {
                node[iatnode] {$s$} 
                    child {
                        node[endnotnecessary, label=right:
                            {$\dots$}] {}
                        edge from parent[notnecessary]
                        node[above] { }
                        node[below] { }
                    }
                    child {
                        node[end, label=right:
                            {$s=1$}] {}
                        edge from parent[lookahead]
                        node[above, fill=white] { }
                        node[below, fill=white] {$a, P^a_{s,1}$}
                    }
                    child {
                        node[endnotnecessary, label=right:
                            {$\dots$}] {}
                        edge from parent[notnecessary]
                        node[above] { }
                        node[below] { }
                    }
                    child {
                        node[endnotnecessary, label=right:
                            {$\dots$}] {}
                        edge from parent[notnecessary]
                        node[above] { }
                        node[below] { }
                    }
                edge from parent
                node[above] { }
                node[below]  {$a=1, P^1_{1,0}$}
            }
            child {
                node[iatnode] {$s$} 
                    child {
                        node[endnotnecessary, label=right:
                            {$\dots$}] {}
                        edge from parent[notnecessary]
                        node[above] { }
                        node[below] { }
                    }
                    child {
                        node[end, label=right:
                            {$s=1$}] {}
                        edge from parent[lookahead]
                        node[above, fill=white] { }
                        node[below, fill=white] {$a, P^a_{s,1}$}
                    }
                    child {
                        node[endnotnecessary, label=right:
                            {$\dots$}] {}
                        edge from parent[notnecessary]
                        node[above] { }
                        node[below] { }
                    }
                    child {
                        node[endnotnecessary, label=right:
                            {$\dots$}] {}
                        edge from parent[notnecessary]
                        node[above] { }
                        node[below] { }
                    }
                edge from parent
                node[above] { }
                node[below]  {$a=0, P^0_{1,0}$}
            }
            edge from parent 
            node[above] { }
            node[below]  {$a=1, P^1_{s_0, 1}$}
    }
    child {
        node[bag] {$s=1$}        
            child {
                node[iatnode] {$s$} 
                    child {
                        node[endnotnecessary, label=right:
                            {$\dots$}] {}
                        edge from parent[notnecessary]
                        node[above] { }
                        node[below] { }
                    }
                    child {
                        node[end, label=right:
                            {$s=1$}] {}
                        edge from parent[lookahead]
                        node[above,fill=white] { }
                        node[below,fill=white] {$a, P^a_{s,1}$}
                    }
                    child {
                        node[endnotnecessary, label=right:
                            {$\dots$}] {}
                        edge from parent[notnecessary]
                        node[above] { }
                        node[below] { }
                    }
                    child {
                        node[endnotnecessary, label=right:
                            {$\dots$}] {}
                        edge from parent[notnecessary]
                        node[above] { }
                        node[below] { }
                    }
                edge from parent
                node[above] { }
                node[below]  {$a=1, P^1_{1,1}$}
            }
            child {
                node[iatnode] {$s$} 
                    child {
                        node[endnotnecessary, label=right:
                            {$\dots$}] {}
                        edge from parent[notnecessary]
                        node[above] { }
                        node[below] { }
                    }
                    child {
                        node[end, label=right:
                            {$s=1$}] {}
                        edge from parent[lookahead]
                        node[above, fill=white] { }
                        node[below, fill=white] {$a, P^a_{s,1}$}
                    }
                    child {
                        node[endnotnecessary, label=right:
                            {$\dots$}] {}
                        edge from parent[notnecessary]
                        node[above] { }
                        node[below] { }
                    }
                    child {
                        node[endnotnecessary, label=right:
                            {$\dots$}] {}
                        edge from parent[notnecessary]
                        node[above] { }
                        node[below] { }
                    }
                edge from parent
                node[above] { }
                node[below]  {$a=0, P^0_{1,1}$}
            }
            child {
                node[iatnode] {$s$} 
                    child {
                        node[endnotnecessary, label=right:
                            {$\dots$}] {}
                        edge from parent[notnecessary]
                        node[above] { }
                        node[below] { }
                    }
                    child {
                        node[end, label=right:
                            {$s=1$}] {}
                        edge from parent[lookahead]
                        node[above, fill=white] { }
                        node[below, fill=white] {$a, P^a_{s,1}$}
                    }
                    child {
                        node[endnotnecessary, label=right:
                            {$\dots$}] {}
                        edge from parent[notnecessary]
                        node[above] { }
                        node[below] { }
                    }
                    child {
                        node[endnotnecessary, label=right:
                            {$\dots$}] {}
                        edge from parent[notnecessary]
                        node[above] { }
                        node[below] { }
                    }
                edge from parent
                node[above] { }
                node[below]  {$a=1, P^1_{1,0}$}
            }
            child {
                node[iatnode] {$s$} 
                    child {
                        node[endnotnecessary, label=right:
                            {$\dots$}] {}
                        edge from parent[notnecessary]
                        node[above] { }
                        node[below] { }
                    }
                    child {
                        node[end, label=right:
                            {$s=1$}] {}
                        edge from parent[lookahead]
                        node[above, fill=white] { }
                        node[below, fill=white] {$a, P^a_{s,1}$}
                    }
                    child {
                        node[endnotnecessary, label=right:
                            {$\dots$}] {}
                        edge from parent[notnecessary]
                        node[above] { }
                        node[below] { }
                    }
                    child {
                        node[endnotnecessary, label=right:
                            {$\dots$}] {}
                        edge from parent[notnecessary]
                        node[above] { }
                        node[below] { }
                    }
                edge from parent
                node[above] { }
                node[below]  {$a=0, P^0_{1,0}$}
            }
            edge from parent 
            node[above] { }
            node[below]  {$a=0, P^0_{s_0, 1}$}
    }
    child {
        node[bag] {$s=0$}        
            child {
                node[iatnode] {$s$} 
                    child {
                        node[endnotnecessary, label=right:
                            {$\dots$}] {}
                        edge from parent[notnecessary]
                        node[above] { }
                        node[below] { }
                    }
                    child {
                        node[end, label=right:
                            {$s=1$}] {}
                        edge from parent[lookahead]
                        node[above,fill=white] { }
                        node[below,fill=white] {$a, P^a_{s,1}$}
                    }
                    child {
                        node[endnotnecessary, label=right:
                            {$\dots$}] {}
                        edge from parent[notnecessary]
                        node[above] { }
                        node[below] { }
                    }
                    child {
                        node[endnotnecessary, label=right:
                            {$\dots$}] {}
                        edge from parent[notnecessary]
                        node[above] { }
                        node[below] { }
                    }
                edge from parent
                node[above] { }
                node[below]  {$a=1, P^1_{0,1}$}
            }
            child {
                node[iatnode] {$s$} 
                    child {
                        node[endnotnecessary, label=right:
                            {$\dots$}] {}
                        edge from parent[notnecessary]
                        node[above] { }
                        node[below] { }
                    }
                    child {
                        node[end, label=right:
                            {$s=1$}] {}
                        edge from parent[lookahead]
                        node[above, fill=white] { }
                        node[below, fill=white] {$a, P^a_{s,1}$}
                    }
                    child {
                        node[endnotnecessary, label=right:
                            {$\dots$}] {}
                        edge from parent[notnecessary]
                        node[above] { }
                        node[below] { }
                    }
                    child {
                        node[endnotnecessary, label=right:
                            {$\dots$}] {}
                        edge from parent[notnecessary]
                        node[above] { }
                        node[below] { }
                    }
                edge from parent
                node[above] { }
                node[below]  {$a=0, P^0_{0,1}$}
            }
            child {
                node[iatnode] {$s$} 
                    child {
                        node[endnotnecessary, label=right:
                            {$\dots$}] {}
                        edge from parent[notnecessary]
                        node[above] { }
                        node[below] { }
                    }
                    child {
                        node[end, label=right:
                            {$s=1$}] {}
                        edge from parent[lookahead]
                        node[above, fill=white] { }
                        node[below, fill=white] {$a, P^a_{s,1}$}
                    }
                    child {
                        node[endnotnecessary, label=right:
                            {$\dots$}] {}
                        edge from parent[notnecessary]
                        node[above] { }
                        node[below] { }
                    }
                    child {
                        node[endnotnecessary, label=right:
                            {$\dots$}] {}
                        edge from parent[notnecessary]
                        node[above] { }
                        node[below] { }
                    }
                edge from parent
                node[above] { }
                node[below]  {$a=1, P^1_{0,0}$}
            }
            child {
                node[iatnode] {$s$} 
                    child {
                        node[endnotnecessary, label=right:
                            {$\dots$}] {}
                        edge from parent[notnecessary]
                        node[above] { }
                        node[below] { }
                    }
                    child {
                        node[end, label=right:
                            {$s=1$}] {}
                        edge from parent[lookahead]
                        node[above, fill=white] { }
                        node[below, fill=white] {$a, P^a_{s,1}$}
                    }
                    child {
                        node[endnotnecessary, label=right:
                            {$\dots$}] {}
                        edge from parent[notnecessary]
                        node[above] { }
                        node[below] { }
                    }
                    child {
                        node[endnotnecessary, label=right:
                            {$\dots$}] {}
                        edge from parent[notnecessary]
                        node[above] { }
                        node[below] { }
                    }
                edge from parent
                node[above] { }
                node[below]  {$a=0, P^0_{0,0}$}
            }
            edge from parent 
            node[above] { }
            node[below]  {$a=1, P^1_{s_0, 0}$}
    }
    child {
        node[bag] {$s=0$}        
            child {
                node[iatnode] {$s$} 
                    child {
                        node[endnotnecessary, label=right:
                            {$\dots$}] {}
                        edge from parent[notnecessary]
                        node[above] { }
                        node[below] { }
                    }
                    child {
                        node[end, label=right:
                            {$s=1$}] {}
                        edge from parent[lookahead]
                        node[above,fill=white] { }
                        node[below,fill=white] {$a, P^a_{s,1}$}
                    }
                    child {
                        node[endnotnecessary, label=right:
                            {$\dots$}] {}
                        edge from parent[notnecessary]
                        node[above] { }
                        node[below] { }
                    }
                    child {
                        node[endnotnecessary, label=right:
                            {$\dots$}] {}
                        edge from parent[notnecessary]
                        node[above] { }
                        node[below] { }
                    }
                edge from parent
                node[above] { }
                node[below]  {$a=1, P^1_{0,1}$}
            }
            child {
                node[iatnode] {$s$} 
                    child {
                        node[endnotnecessary, label=right:
                            {$\dots$}] {}
                        edge from parent[notnecessary]
                        node[above] { }
                        node[below] { }
                    }
                    child {
                        node[end, label=right:
                            {$s=1$}] {}
                        edge from parent[lookahead]
                        node[above, fill=white] { }
                        node[below, fill=white] {$a, P^a_{s,1}$}
                    }
                    child {
                        node[endnotnecessary, label=right:
                            {$\dots$}] {}
                        edge from parent[notnecessary]
                        node[above] { }
                        node[below] { }
                    }
                    child {
                        node[endnotnecessary, label=right:
                            {$\dots$}] {}
                        edge from parent[notnecessary]
                        node[above] { }
                        node[below] { }
                    }
                edge from parent
                node[above] { }
                node[below]  {$a=0, P^0_{0,1}$}
            }
            child {
                node[iatnode] {$s$} 
                    child {
                        node[endnotnecessary, label=right:
                            {$\dots$}] {}
                        edge from parent[notnecessary]
                        node[above] { }
                        node[below] { }
                    }
                    child {
                        node[end, label=right:
                            {$s=1$}] {}
                        edge from parent[lookahead]
                        node[above, fill=white] { }
                        node[below, fill=white] {$a, P^a_{s,1}$}
                    }
                    child {
                        node[endnotnecessary, label=right:
                            {$\dots$}] {}
                        edge from parent[notnecessary]
                        node[above] { }
                        node[below] { }
                    }
                    child {
                        node[endnotnecessary, label=right:
                            {$\dots$}] {}
                        edge from parent[notnecessary]
                        node[above] { }
                        node[below] { }
                    }
                edge from parent
                node[above] { }
                node[below]  {$a=1, P^1_{0,0}$}
            }
            child {
                node[iatnode] {$s$} 
                    child {
                        node[endnotnecessary, label=right:
                            {$\dots$}] {}
                        edge from parent[notnecessary]
                        node[above] { }
                        node[below] { }
                    }
                    child {
                        node[end, label=right:
                            {$s=1$}] {}
                        edge from parent[lookahead]
                        node[above, fill=white] { }
                        node[below, fill=white] {$a, P^a_{s,1}$}
                    }
                    child {
                        node[endnotnecessary, label=right:
                            {$\dots$}] {}
                        edge from parent[notnecessary]
                        node[above] { }
                        node[below] { }
                    }
                    child {
                        node[endnotnecessary, label=right:
                            {$\dots$}] {}
                        edge from parent[notnecessary]
                        node[above] { }
                        node[below] { }
                    }
                edge from parent
                node[above] { }
                node[below]  {$a=0, P^0_{0,0}$}
            }
            edge from parent 
            node[above] { }
            node[below]  {$a=0, P^0_{s_0, 0}$}
    };
\node[draw=RoyalBlue, dashed, thick,fit={(root-1-1) (root-4-4)},label=above:{$t$}] { }; 
\end{tikzpicture}
\end{figure}

The number of decision variables required to enumerate each arm's game tree is of order $O(|\mathcal{A}||\mathcal{S}|^T)$ and there are $N$ such trees, so even a linear program (LP) relaxation is not tractable for larger values of $T$ and $N$, which motivates us to propose \textsc{ProbFair} (Section \ref{sec:stationary-policy}) as an efficient alternative.  

\begin{example}
Suppose we wish to find the coefficient \(c'\) corresponding to \(x_{i,a=0,t=2}\). From Equation~\ref{eqn:IPcoeffs}, we have \(c' = \frac{1}{2^2}\sum_{s \in S} p(s_2 = s)p(s_{3} = 1 | x_{i,a=0,t=2}, s_2 = s)\). Equivalently, we sum the weight of each path from the root node to the highlighted end nodes in Figure~\ref{tikz:IPcoeffs} and normalize by \(\frac{1}{2^2}\):
\begin{align}
    c' &= \frac{1}{4}\left(P^0_{s_0,0}P^0_{0,0}P^0_{0,1}+ P^0_{s_0,0}P^1_{0,0}P^0_{0,1} + P^0_{s_0,0}P^0_{0,1}P^0_{1,1} + P^0_{s_0,0}P^1_{0,1}P^0_{1,1} \right.\\
    &+ P^1_{s_0,0}P^0_{0,0}P^0_{0,1}+ P^1_{s_0,0}P^1_{0,0}P^0_{0,1} + P^1_{s_0,0}P^0_{0,1}P^0_{1,1} + P^1_{s_0,0}P^1_{0,1}P^0_{1,1} \nonumber \\
    &+ \left. \dots\right) \quad \textrm{For each of the \(\left(\lvert \mathcal{A}\rvert \lvert\mathcal{S}\rvert\right)^t = 16\) paths to a blue node in Figure~\ref{tikz:IPcoeffs}.} \nonumber
\end{align}
\end{example}

\subsubsection{Comparison of \textsc{ProbFair} with the True Optimal Policy}
\label{sec:AppAddlExpIP}
In Section~\ref{sec:experimentalEvaluation}, we normalize intervention benefit with \textsc{Threshold Whittle}, which is asymptotically optimal for forward threshold-optimal transition matrices under a budget constraint \(k\)~\citep{mate2020collapsing}. However, with the integer program (IP) we formulate in Section~\ref{sec:ip}, we can find the \textit{optimal} policy for any set of transition matrices under budget \emph{and} fairness constraints as long as \(N\) and \(T\) are small. % \todo{do these refs go to the correct places?}

\begin{figure}[!h]
  \centering
  \includegraphics[width=0.9\linewidth]{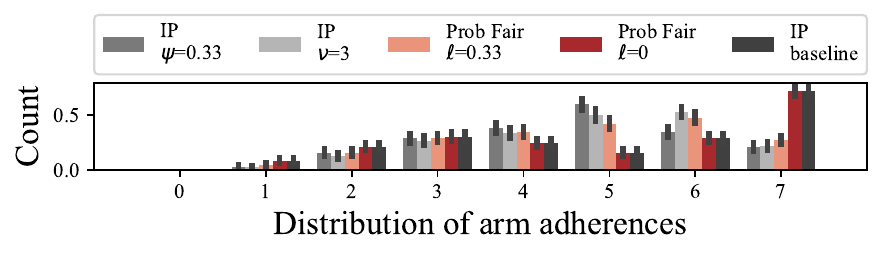}
  \caption{Adherences of \textsc{ProbFair}, compared to the IP formulation}
\label{fig:stationaryvsIP_adherences}
\end{figure}

We generate \(N=2\) random arms such that the structural constraints outlined in Section~\ref{sec:rmabModel} are satisfied. We set \(k=1\) and \(T=6\). Though the variance in reward is large due to the small \(T\), Figure~\ref{fig:stationaryvsIP_adherences} shows that \textsc{ProbFair} obtains 100\% of the intervention benefit when no fairness constraints are applied. Similarly, \textsc{ProbFair} with \(\ell = 0.33\) obtains the same adherence behavior as the IP policy with under hard fairness constraint \(\nu=3\) or minimum selection fraction constraint \(\psi = 0.33\). (within 95\% confidence interval shown).  All results shown are bootstrapped over 500 iterations.

\paragraph{Minimum Selection Fraction Constraints.}
\label{sec:AppMinSelFrac}
As we discuss in Appendix~\ref{sec:AppWhittleUnFair}, the optimal policy is often to pull the same \(k\) arms at every timestep and ignore all other arms. Under minimum selection fraction constraints (Equation~\ref{eqn:minSelFrac}), each arm must be pulled at least a minimum fraction \(\psi\) of \(T\) rounds, with no conditions on when these pulls should take place. We confirm with the optimal IP implementation our intuition that these additional pulls are allocated at the beginning or end of the simulation. That is, the optimal policy under minimum selection fraction constraints is to take advantage of the finite time horizon, which is not suitable for the applications we consider.

\clearpage

\section{\textsc{ProbFair}: a Probablistically Fair Policy}
\label{app:appProbFair}
Our main contribution is the novel probabilistic policy \textsc{ProbFair} (see Section~\ref{sec:algApproach}). Here, in Section~\ref{sec:appProbFairProofs},
we present complete proofs to Theorems \ref{thm:fcurvature}-\ref{thm:solveP2}. Then, in Section \ref{sec:AppProbFairSampling}, we provide additional details about how we sample from our probabilistic policy to select discrete actions at each timestep.

\subsection{Proofs}\label{sec:appProbFairProofs}
In this section, we provide proofs for the theorems introduced in Section~\ref{sec:algApproach}. When relevant, we begin by restating the theorem for convenience.

\fcurvature*

\begin{proof} For notational convenience, let: 
\begin{align*}
    c_1 &= P_{0,1}^0; \\
    c_2 &= P_{0,1}^1 - P_{0,1}^0; \\
    c_3 &= 1 - P_{1,1}^0 + P_{0,1}^0; \\
    c_4 &= P_{1,1}^0 - P_{1,1}^1 - P_{0,1}^0 + P_{0,1}^1.  
\end{align*}
Then, $f_i(p_i) = \frac{c_1 + c_2p_i}{c_3 + c_4p_i}$. We observe that $\forall i \in [N], f_i(p_i)$ is a valid probability since the term $1-(1-p_i)P_{1,1}^0-p_iP_{1,1}^1$ in the denominator is at least $1-(1-p_i) -p_i = 0$ for all $p_i \in [0,1]$. Then, there are three cases which describe the possible shapes of $f_i(p_i)$:\\
\begin{case}
 $c_4$ = 0. Here, $f_i(p_i)$ is \textbf{linear} and hence, \textbf{concave}. 
\end{case}

\begin{case}
$c_4 \neq 0; c_2 = 0$. Here, $f_i^{\dprime}(p_i) = \frac{2c_1c_4^2}{(c_3 + c_4p_i)^3} \geq 0$, so $f_i(p_i)$ is linear (hence \textbf{concave}) if $c_1 = 0$ or \textbf{strictly convex} (if $c_1 > 0$) in the domain $p_i \in [0,1]$. 
\end{case}

\begin{case}
$c_4 \neq 0; c_2 \neq 0$. Here, 
\begingroup
\setlength\abovedisplayskip{3pt}
\setlength\belowdisplayskip{3pt}
\begin{equation}
f_i(p_i) = \frac{\frac{c_2}{c_4}\left(\frac{c_1c_4}{c_2} + c_4p_i\right)}{c_3 + c_4p_i} 
        =  \frac{c_2}{c_4} + \frac{\left(c_1 - \frac{c_2c_3}{c_4}\right)}{c_3 + c_4p_i}.
\end{equation}
\endgroup
Thus, 
\begin{equation}
    f_i^{\dprime}(p_i) = \frac{2c_4^2 \left(c_1 - \frac{c_2c_3}{c_4}\right)}{(c_3 + c_4p_i)^3}
\end{equation}
The sign of \(f_i^{\dprime}(p_i)\) is the same as the sign of $d = c_1 - \frac{c_2c_3}{c_4}$. It follows that $f_i(p_i)$ is \textbf{strictly convex} if $c_1 >  \frac{c_2c_3}{c_4}$, and \textbf{concave} otherwise for $p_i \in [0,1]$.
%~\qedsymbol 
\end{case}
\end{proof}

\fnondecr*

\begin{proof} For notational convenience, let: 
\begin{align*}
    c_1 &= P_{0,1}^0; \\
    c_2 &= P_{0,1}^1 - P_{0,1}^0; \\
    c_3 &= 1 - P_{1,1}^0 + P_{0,1}^0; \\
    c_4 &= P_{1,1}^0 - P_{1,1}^1 - P_{0,1}^0 + P_{0,1}^1.  % a_{0,1} - a_{1,1} - a_{0,0} + a_{1,0}
\end{align*}

Then \(f_i(p_i) = \frac{c_1+c_2p_i}{c_3+c_4p_i}\) and \(f_i^{\prime}(p_i) = \frac{c_2c_3 - c_1c_4}{(c_3 + c_4p_i)^2}\).\\

Observe \(c_2c_3 - c_1c_4  \geq 0\) implies \(f_i^{\prime}(p_i) \geq 0\).
\begin{align*}
c_2c_3 - c_1c_4 &\geq 0 \\
c_2c_3 &\geq c_1c_4\\
(P_{0,1}^1 - P_{0,1}^0)(1 - P_{1,1}^0 + P_{0,1}^0) &\geq P_{0,1}^0(P_{1,1}^0 - P_{1,1}^1 - P_{0,1}^0 + P_{0,1}^1)\\
(1-P_{1,1}^0)P_{0,1}^1 &\geq (1-P_{1,1}^1)P_{0,1}^0
\end{align*}
Per our structural constraints, $P_{1,1}^0 < P_{1,1}^1$ and $P_{0,1}^1 > P_{0,1}^0$.
\end{proof}

\Yoptonboundary*

\begin{proof}
Note by compactness that \textbf{P2} has not just a supremum, but an actual maximum solution. Suppose for contradiction there is some optimal solution $\vec{p}$ with distinct indices $i, j \in \mathcal{Y}$ such that $p_i, p_j \in (\ell, u)$. Now, suppose we perturb by an infinitesimal $\epsilon$ (of arbitrary sign but tiny positive absolute value) such that $p_i := p_i + \epsilon$ and $p_j := p_j - \epsilon$. This satisfies all our constraints for small-enough $|\epsilon|$. The change in the objective $\sum_{i \in \mathcal{Y}} f_i(p_i)$ is now $\epsilon \cdot (f_i'(p_i) - f_j'(p_j)) + O(\epsilon^2)$; hence, if $f_i'(p_i) - f_j'(p_j)$ is nonzero, then we can take a tiny $\epsilon$ of the appropriate sign to increase the objective, a contradiction. Therefore, $f_i'(p_i) - f_j'(p_j) = 0$, and so, we now focus on lower-order terms: the change in the objective $\sum_{i \in \mathcal{Y}} f_i(p_i)$ is now $(\epsilon^2/2) \cdot (f_i''(p_i) + f_j''(p_j)) + O(\epsilon^3)$. However, since $f_i$ and $f_j$ are strictly convex, we have that $f_i''(p_i) + f_j''(p_j) > 0$, and hence the objective increases regardless of the sign of (the tiny) $\epsilon$, again a contradiction. Thus we have our structural result. 
\end{proof}

\solveconvex*

We begin by introducing Lemma~\ref{lem:P2_p_prime_unique}, which we use in our proof of Theorem~\ref{thm:solveP2}:
\begin{lemma}\label{lem:P2_p_prime_unique}
For a given $\gamma$, $p^\prime \in (\ell, \uuu]$. 
\end{lemma}
\begin{proof}
To begin, observe that $z-k = \gamma \ell + (|\mathcal{Y}| - 1 - \gamma)\uuu + p^\prime$. Then, to prove the lower bound, observe that:
\begin{align*}
\gamma &= \left\lfloor \frac{|\mathcal{Y}|\uuu - (k-z)}{\uuu - \ell} \right\rfloor\\
\gamma &> \frac{|\mathcal{Y}|\uuu - (k-z)}{\uuu - \ell} -1 \\
\rightarrow \gamma &> \frac{|\mathcal{Y}|\uuu - (\gamma \ell + (|\mathcal{Y}| - 1 - \gamma)\uuu + p^\prime)}{\uuu - \ell} - 1\\
\gamma &> \frac{|\mathcal{Y}|\uuu - \gamma \ell - |\mathcal{Y}|\uuu + \uuu + \gamma \uuu -p^\prime - \uuu + \ell}{\uuu - \ell}\\
0 &> \ell - p^\prime \implies p^\prime > \ell 
\end{align*}
To prove the upper bound, observe that:
\begin{align*}
\gamma &= \left\lfloor \frac{|\mathcal{Y}|\uuu - (k-z)}{\uuu - \ell} \right\rfloor\\
\gamma &\leq \frac{|\mathcal{Y}|\uuu - (k-z)}{\uuu - \ell}\\
\rightarrow \gamma &\leq \frac{|\mathcal{Y}|\uuu - (\gamma \ell + (|\mathcal{Y}| - 1 - \gamma)\uuu + p^\prime)}{\uuu - \ell}\\
\gamma &\leq \frac{\gamma(\uuu - \ell) + \uuu - p^\prime}{\uuu-\ell}\\
0 &\leq \uuu - p^\prime \implies p^\prime \leq \uuu\\
\end{align*}
Thus, $\ell < p^\prime \leq \uuu$.
\end{proof}

Now, we proceed with our proof of \textbf{Theorem~\ref{thm:solveP2}}.

\begin{proof}
Recall \textbf{P2}: maximize $\sum_{i \in \mathcal{Y}} f_i(p_i)$ such that $p_i \in [\ell, \uuu]$ for all $i \in \mathcal{Y}$ and $\sum_{i \in \mathcal{Y}} p_i = k - z$. By \textbf{Lemma \ref{lem:P2_opt_on_boundary}}, there exists \emph{at most one arm} with optimal value \(p_i^*\in (\ell, u)\). 

First, we discuss an edge case. If \(k-z = \lvert \mathcal{Y}\rvert \ell\), Line 2 of Algorithm~\ref{alg:solveP2} assigns \(\gamma = \lvert\mathcal{Y}\rvert\), so Line 11 assigns 
\begin{align}
    \pi_{\mathcal{Y}} \coloneqq&~ i \mapsto \left.
  \begin{cases}
    \ell, & \text{for } i \in \mathcal{Y}_1 = \{\mathcal{Y}\} \\ 
    p^\prime, & \text{for } i \in \mathcal{Y}_2 = \emptyset \\
    \uuu, & \text{for } i \in \mathcal{Y}_3 = \emptyset\\
  \end{cases}
  \right\}
\end{align}
Thus Algorithm~\ref{alg:solveP2} returns the only valid solution to \textbf{P2} in this case, which is to set \(p_i = \ell\) for all arms \(i \in \mathcal{Y}\). 

For all other cases \(k-z > \lvert \mathcal{Y}\rvert \ell\), we introduce the following notation: let \(\mathcal{Y}_1\) be the set of arms for which \(p_i = \ell\), \(\mathcal{Y}_2\) be a set containing exactly one arm (WLOG \(j\)) where \(p_j = p'\in (\ell, u]\), and \(\mathcal{Y}_3\) be the remaining set of arms for which \(p_i = \uuu\), with $\bigcap_{x=1}^{3} \mathcal{Y}_{x} = \emptyset$. Then by \textbf{Lemma \ref{lem:P2_p_prime_unique}}, \(\gamma = \lvert \mathcal{Y}_1\rvert=\left\lfloor \frac{|\mathcal{Y}|\uuu - (k-z)}{\uuu - \ell} \right\rfloor\) and \(p' = k - z - \gamma\ell - (|\mathcal{Y}|-1-\gamma)\uuu \in (\ell, \uuu]\).

\textbf{P2} is then equivalent to finding a partition \(\mathcal{Y} \to \mathcal{Y}_1 \cup \mathcal{Y}_2 \cup \mathcal{Y}_3\) which maximizes the following: 
\begin{align}
    \argmax_{\left\{\mathcal{Y}_1, \mathcal{Y}_2, \mathcal{Y}_3\right\}} ~ &\sum_{i\in \mathcal{Y}_1} f_i(\ell) + f_{j}(p')+ \sum_{i'' \in \mathcal{Y}_3} f_{i''}(u) \nonumber \\
    \textrm{s.t.}~&\lvert \mathcal{Y}_1 \rvert = \gamma,~\mathcal{Y}_2 = \{j\}, \nonumber \\
    &\bigcap_{x=1}^{3} \mathcal{Y}_{x} = \emptyset  ,~\textrm{and}~\bigcup_{x=1}^3 \mathcal{Y}_{x} = \mathcal{Y} \label{eqn:p2_reorg}
\end{align}

Subtracting the constant \(\sum_{i\in \mathcal{Y}}f_i(\ell)\) and simplifying yields:
\begin{align}
    \argmax_{\left\{\mathcal{Y}_1, \mathcal{Y}_2, \mathcal{Y}_3\right\}} ~& f_{j}(p')-f_j(\ell)+ \sum_{i'' \in \mathcal{Y}_3} f_{i''}(u)-f_{i''}(\ell) \nonumber \\
    \textrm{s.t.}~&\lvert \mathcal{Y}_1 \rvert = \gamma,~\mathcal{Y}_2 = \{j\}, \nonumber \\
    &\bigcap_{x=1}^{3} \mathcal{Y}_{x} = \emptyset  ,~\textrm{and}~\bigcup_{x=1}^3 \mathcal{Y}_{x} = \mathcal{Y}
\end{align}

Suppose we sort arms in ascending order by \(f_i(u) - f_i(\ell)\). Let us create the set \(\mathcal{Y}_3'\) from the last \(\lvert \mathcal{Y}\rvert -\gamma -1\) arms. % \todo{edit sentence, awk.}

By monotonicity, for all \(i \in \mathcal{Y}_3'\) and \(j \not \in \mathcal{Y}_3'\),
\begin{align}
    f_j(p') - f_j(\ell) \leq f_i(u) - f_i(\ell)
\end{align}
%\pagebreak

Thus, setting \(\mathcal{Y}_3^* = \mathcal{Y}_3'\) reduces the optimization problem in Eq.~\ref{eqn:p2_reorg} to finding a partition over the remaining sets \(\mathcal{Y}_1\) and \(\mathcal{Y}_2\).
\begin{align}
    \argmax_{\left\{\mathcal{Y}_1, \mathcal{Y}_2\right\}} ~& f_{j}(p')-f_j(\ell) \nonumber \\
    \textrm{s.t.}~&\lvert \mathcal{Y}_1 \rvert = \gamma,~\mathcal{Y}_2 = \{j\}, \nonumber \\
    &\bigcap_{x=1}^{2} \mathcal{Y}_{x} = \emptyset  ,~\textrm{and}~\bigcup_{x=1}^2 \mathcal{Y}_{x} = \mathcal{Y}\setminus \mathcal{Y}_3^* \label{eqn:p2_finalstage}
\end{align}

Finally, we solve Equation~\ref{eqn:p2_finalstage} by finding the arm \(j\) with maximal value \(f_j(p') - f_j(\ell)\). Then \(\mathcal{Y}_2^* = {j}\), \(\mathcal{Y}_1^* = \mathcal{Y}\setminus \left(\mathcal{Y}_2 \bigcup\mathcal{Y}_3^*\right)\), and we are done.
\end{proof}

\complexityconvex*
At worst, Algorithm \ref{alg:solveP2} requires two sorts: once on Line 5, and a second time on Line 8, for a total computational cost of $O(2|\mathcal{Y}| \log |\mathcal{Y}|)$. In total, the computational cost of Algorithm \ref{alg:probfair} is at worst \(O\!\left(\frac{kN}{\epsilon^3}\right)\) when all \(N\) arms are in \(\mathcal{X}\). 

\subsection{Dependent Rounding Sampling Approach}
\label{sec:AppProbFairSampling}
Here we provide pseudocode for the sampling algorithm introduced in Section \ref{sec:sampling}, along with its associated \textsc{Simplify} subroutine~\citep{srinivasan2001distributions}.

\setlength{\textfloatsep}{0.1cm}
\setlength{\floatsep}{0.1cm}
\begin{algorithm}[H]
\caption{Sampling Subroutine (adapted from ~\citet{srinivasan2001distributions})}
\label{alg:simplify}
\begin{algorithmic}[1] 
\Procedure{Simplify}{$\alpha \in [0,1], \beta \in [0,1]$}\label{fig:alg2}
\If{$\alpha = \beta = 0$}
    \State{$p_i,  p_j \gets [0,0]$}
\ElsIf{$\alpha = \beta = 1$} 
    \State{$p_i,  p_j \gets [1,1]$}
\ElsIf{$\alpha + \beta = 1$} 
    \State{\texttt{flag} $\gets X \sim B(n=1, p=\alpha)$}
    \State{$p_i, p_j \gets [1,0]$ if \texttt{flag} else $[0,1]$}
\ElsIf{$0 < \alpha + \beta < 1$}
    \State{\texttt{flag} $\gets X \sim B\!\left(n=1, p=\frac{\alpha}{\alpha + \beta}\right)$}
    \State{$p_i, p_j \gets [\alpha + \beta,0]$ if \texttt{flag} else $[0,\alpha + \beta]$}
\ElsIf{$1 < \alpha + \beta < 2$}
     \State{\texttt{flag} $\gets X \sim B\!\left(n=1, p=\frac{1-\beta}{2- \alpha - \beta}\right)$}
     \State{$p_i, p_j \gets [1, \alpha + \beta -1]$ if \texttt{flag} else $[\alpha + \beta -1, 1]$}
\EndIf{}
\Return{$p_i, p_j$}
\EndProcedure
\end{algorithmic}
\end{algorithm}
\setlength{\textfloatsep}{0.1cm}
\setlength{\floatsep}{0.1cm}

\setlength{\textfloatsep}{0.1cm}
\setlength{\floatsep}{0.1cm}
\begin{algorithm}[H]
\caption{Sampling Algorithm (adapted from~\citet{srinivasan2001distributions})}
\begin{algorithmic}[1] 
\Procedure{Sample}{$G = (V,E)$}\label{fig:alg3}
\State{$H \gets G \setminus \{v | \exists e \in G \text{ s.t. } e_{\text{dst}} = v\}$} \Comment{\textcolor{RoyalBlue}{Select subgraph containing nodes without a parent}}
\If{$|H| = 1$}
\Return{$G$}\Comment{\textcolor{RoyalBlue}{$z_v \in \{0,1\} \forall v \in G; \sum_v z_v = k$}}
\ElsIf{$|H| \geq 2$}
\State{$A \subsetneq G \in {H \choose \lfloor{\frac{|H|}{2}\rfloor}}$}
\State{$B \gets G \setminus A$}
\State{\texttt{pairs} $\gets \{(a_i,b_i) \in A \times B | i \in \mathcal{I}\}$}
\State{$H^{\prime} \gets (V =\emptyset, E = \emptyset)$}
\For{$(v_i, v_j) \in \texttt{pairs}$}
\State{$H^{\prime} \gets H^{\prime} \cup \{v_i, v_j\}$}
\State{$H^{\prime} \gets H^{\prime} \cup \{v^\prime; e_{v^\prime,v_\alpha} | \alpha \in \{i,j\} \}$}
\State{$X_i, X_j \gets \Call{SIMPLIFY}{p_{v_i}, p_{v_j}}$} \Comment{\textcolor{RoyalBlue}{Defined in Algorithm~\ref{alg:simplify}}}
\State{$z_{v_i} \gets X_i$} \Comment{\textcolor{RoyalBlue}{If $X_i$ was fixed, $z_{v_i} \in \{0,1\}$}}
\State{$z_{v_j} \gets X_j$}\Comment{\textcolor{RoyalBlue}{If $X_j$ was fixed, $z_{v_j} \in \{0,1\}$}}
\If{$z_{v_i} \in \{0,1\}$}
\State{$p_{v^\prime} \gets X_j$} 
\Else{$p_{v^\prime} \gets X_i$}
\EndIf 
\EndFor
\State{$F \gets G \cup H^\prime$}\Comment{\textcolor{RoyalBlue}{$\forall v \in G \cap H^\prime$, update attribute values per $H^\prime$}}
\State\Return{\Call{SAMPLE}{F}}
\EndIf 
\EndProcedure
\end{algorithmic}
\end{algorithm}
\setlength{\textfloatsep}{0.1cm}
\setlength{\floatsep}{0.1cm}
\clearpage

\section{Additional Experimental Details}
In this section, we discuss additional details of our empirical study in Section~\ref{sec:experimentalEvaluation}. We provide a description and pseudocode of the heuristic policies (\ref{sec:AppHeuristics}), discuss our choice of fairness metric (\ref{sec:AppExpMetrics}), and provide additional details of our Synthetic dataset (\ref{sec:appSynthdata}).

Code and instructions needed to reproduce these experimental results are included at: \url{https://github.com/crherlihy/prob_fair_rmab}.
All results presented in this paper are bootstrapped over 100 simulation iterations, with a time horizon $T=180$, cohort size $N=100$, and budget $k=20$, unless otherwise noted. We utilize seeds to ensure reproducible variation for each randomized parameter, including actualized transitions in each simulation.
We have run simulations on an Intel(R) Core i7 CPU with 16Gb of RAM. Simulations are configurable via configuration files; runs are trivially parallelizable via these configuration files. 

\subsection{Heuristic Algorithms}
\label{sec:AppHeuristics}
In Section~\ref{sec:experimentalEvaluation}, three heuristics based on the \textsc{Threshold Whittle} algorithm are introduced: \textsc{H}$_{\textsc{First}}$ , \textsc{H}$_{\textsc{Last}}$ , and \textsc{H}$_{\textsc{Rand}}$  Here, we go into more detail and provide pseudocode.

\begin{definition}
\label{def:constPull}
{Within the context of Algorithm \ref{alg:heuristics}, we define a \emph{constrained pull} to be one that is executed to satisfy an integer periodicity constraint. Only arms that have not yet been pulled the required number of times within the \(\nu\)-length interval are available; other arms are excluded from consideration, unless \emph{all} arms have already satisfied their constraints. In this case, all arms are available to be pulled.} 
\end{definition}

If a pull is not constrained, we say it is \emph{unconstrained} or \emph{residual}. 

The \textsc{H}$_{\textsc{First}}$  heuristic requires that all constrained pulls must occur at the start of the interval. This implies that the first $N/k$ timesteps in each interval are dedicated to pulling all $N$ arms.

The \textsc{H}$_{\textsc{Last}}$  heuristic requires that all constrained pulls must occur at the end of the interval. Unlike the \textsc{H}$_{\textsc{First}}$ heuristic, not all arms will necessarily be pulled in the last $N/k$ timesteps, as some arms will have already satisfied their constraint earlier in the interval 
% on the 
via unconstrained pull(s). These leftover constrained pulls function as unconstrained pulls, per Definition~\ref{def:constPull}.

The \textsc{H}$_{\textsc{Rand}}$  heuristic chooses random positions within the interval for constrained pulls to occur. Similarly to the \textsc{H}$_{\textsc{Last}}$  heuristic, some of the later constrained pulls may become unconstrained pulls if all arms have already satisfied their constraint earlier in the interval.

\begin{algorithm}[H]
\caption{Periodicity Constraint-Enforcing Heuristic Based on \textsc{Threshold Whittle}}
\label{alg:heuristics}
\begin{algorithmic}[1]
\Procedure{Simulation}{$A$, $T$, $\nu$, $k$}
\For{interval $\in [0, T]$ with step size $\nu$}
    \State $C_{\text{interval}} \gets \emptyset$ \Comment{\textcolor{RoyalBlue}{$C_{\text{interval}} \coloneqq$ arm(s) with constraint satisfied during the interval}}
\EndFor\\
\For{$a \in A$}
    \State $a.\text{last observed state} \gets 1$
    \State $a.\text{time since pull} \gets 1$
\EndFor \\
\For{$t \in T$}
    \State{$i \gets$ \texttt{GetInterval}($t$)}\\
    \If{$t$ is a constrained pull $\land \ C_i \subsetneq A$}
    
    \State $A^{\prime} \gets \{a | a \in A \setminus C_{i}\}$ \Comment{\textcolor{RoyalBlue}{Consider arms with constraint not yet satisfied in interval}}
\ElsIf{$t$ is a residual pull $\lor \ C_i = A$}
    \State $A^{\prime} \gets A$ \Comment{\textcolor{RoyalBlue}{Consider all arms}}
\EndIf\\
\State $A^{\prime}_k \gets$ \texttt{SelectTopK}($A^{\prime}$, $k$, $t$) \Comment{\textcolor{RoyalBlue}{Select $k$ arms with highest Whittle index}}
\State $C_i \gets C_i \cup A^{\prime}_k$\\
\For{$a \in A$}
\State $s_{t+1}(a) \gets$ \texttt{UpdateState}($a$) \Comment{\textcolor{RoyalBlue}{Update each arm's state using belief}}
\EndFor
\Return % $R$
\EndFor 
\EndProcedure
\label{fig:alg1}
\end{algorithmic}
\end{algorithm}

\subsection{Fairness Metric Choices}
\label{sec:AppExpMetrics}
It is not immediately obvious which evaluation metric(s) best indicate whether we have improved distributive fairness. While constraint satisfaction itself is a logical candidate, it is Boolean-valued at the arm level, and thus does not reflect \textit{to what extent} a policy fairly allocates pulls. Even if we were to report population-level constraint satisfaction (i.e., by noting the proportion of arms for which a given fairness constraint is satisfied, either over the course of a single simulation, or in expectation over a set of simulation iterations), this would be tautologically biased in favor of \textsc{ProbFair} and the \textsc{Threshold Whittle}-based heuristics, which explicitly encode constraint satisfaction. This observation motivates us to consider proxy metrics, including the price of fairness (PoF), the Herfindahl–Hirschman Index (HHI), and the earth mover's distance (EMD).

\paragraph{Price of Fairness.} Consider \textit{price of fairness}, defined formally as:
\begin{equation}\label{eq:pof}
    \textsc{PoF}_{\text{TW}}(\text{ALG}) \coloneqq \frac{\mathbb{E}_\text{TW}[R(\cdot)] - \mathbb{E}_\text{ALG}[R(\cdot)]}{\mathbb{E}_\text{TW}[R(\cdot)]}
\end{equation}
Price of fairness is the relative loss in total expected reward associated with following a distributive fairness-enforcing policy, as compared to \textsc{Threshold Whittle}~\citep{bertsimas2011price}. A small loss (\(\sim0\%\)) indicates that fairness has a small impact on total expected reward; conversely, a large loss means total expected reward is sacrificed in order to satisfy the fairness constraints.

\setcounter{thm}{8}
\begin{lemma}\label{thm:IBcorrPoF}

Price of fairness is inversely proportional to intervention benefit.
\end{lemma}

\begin{proof}
The statement in Lemma \ref{thm:IBcorrPoF} is equivalent to the statement ``Given \(y,z>0\), there exists \(\alpha \in \mathbb{R}\) such that \(\frac{x-z}{y-z} = \alpha \frac{z-x}{z}\) for all \(x>0\)". Here \(x=\mathbb{E}_\text{ALG}[R(\cdot)]\), \(y=\mathbb{E}_\text{NoAct}[R(\cdot)]\), and \(z=\mathbb{E}_\text{TW}[R(\cdot)]\). Consider \(\alpha = \frac{-z}{y-z}\). Then \(\alpha \frac{z-x}{z} = \frac{x-z}{y-z}\). Thus, for any algorithm ALG, \(\textsc{PoF}_{\text{TW}}(\text{ALG}) \propto \textsc{IB}_{\text{NoAct}, \text{TW}}(\text{ALG})^{-1}\).
\end{proof}

\paragraph{Herfindahl–Hirschman Index (HHI).}
\label{sec:AppMetricsHHI}
The Herfindahl–Hirschman Index (HHI)~\citep{rhoades1993herfindahl}, is a statistical measure of concentration useful for measuring the extent to which a small set of arms receive a large proportion of attention due to an unequal distribution of scarce pulls~\citep{hirschman1980national}. It is defined as:
\begin{equation}\label{eq:hhi}
    \textsc{HHI}(\text{ALG}) \coloneqq \sum_{i=1}^N \left(\frac{1}{kT} \sum_{t=1}^T a_t^i\right)^2
\end{equation}
HHI ranges from \(1/N\) to 1; higher values indicate that pulls are concentrated on a small subset of arms. However, HHI is an imperfect evaluation metric for addressing our prioritarian concern for arms that would be \textit{deprived} of algorithmic attention (i.e., fail to receive any pulls) under \textsc{Threshold Whittle} (see Appendix~\ref{sec:AppWhittleUnFair}). Since entries are squared, reducing $\uuu$ offers a more direct path to lowering HHI than increasing $\ell$. However, reducing $\uuu$ will not accomplish our stated goal of guaranteeing each arm a strictly positive lower bound on the probability that it will receive a pull at any given timestep. 

\paragraph{Earth Mover's Distance}
The earth mover's distance (EMD), or Wasserstein metric, is a measure of distance between two distributions. Specifically, we measure the distance of an algorithm's distribution of cumulative pull allocations to a fair reference distribution, \textsc{Round-Robin}. Though differences in distances are meaningful, EMD does not directly map to our fairness desiderata. That is, a given level of fairness enforcement (e.g., as characterized by the hyperparameters \(\ell\) or \(\nu\)) is not associated with a specific range of EMD values. Hence, our discussion of (normalized) earth mover's distances in Section~\ref{sec:experimentalEvaluation} focuses on relative differences between policies.

\subsection{Synthetic Dataset}
\label{sec:appSynthdata}

\begin{conjecture}
The set of forward (reverse) threshold-optimal arms are a subset of the set of concave (strictly convex) arms for the local reward function we consider, \(r(s) = s\).
\end{conjecture}

\citet{mate2021risk-aware} provide conditions for threshold optimality. First, the arm must satisfy the structural constraints (Section~\ref{sec:rmabModel}) and the \textit{indexability} condition \(P_{1,1}^0 - P_{0,1}^0 + P_{1,1}^1 - P_{0,1}^1 \leq 1\).
Then, the following inequalities determine forward (reverse) threshold optimality:
\begin{equation}
    \begin{dcases}
     P_{1,1}^0 - P_{0,1}^0 \geq  P_{1,1}^1 - P_{0,1}^1 & \textrm{\textit{forward threshold-optimal}} \\
     P_{1,1}^0 - P_{0,1}^0 \leq  P_{1,1}^1 - P_{0,1}^1 & \textrm{\textit{reverse threshold-optimal}} \\
    \end{dcases}
\end{equation}

We conjecture that these conditions necessarily imply the conditions for concavity, repeated here for convenience:
\begin{equation}
    \begin{dcases}
    P_{0,1}^0 \leq \frac{\left(P_{0,1}^1 - P_{0,1}^0\right)\left(1 - P_{1,1}^0 + P_{0,1}^0\right)}{P_{1,1}^0 - P_{1,1}^1 - P_{0,1}^0 + P_{0,1}^1} &\textrm{\textit{concave}} \\
    P_{0,1}^0 > \frac{\left(P_{0,1}^1 - P_{0,1}^0\right)\left(1 - P_{1,1}^0 + P_{0,1}^0\right)}{P_{1,1}^0 - P_{1,1}^1 - P_{0,1}^0 + P_{0,1}^1} &\textrm{\textit{strictly convex}}
    \end{dcases}
\end{equation}

\end{document}